\documentclass[11pt]{article}

 \pdfoutput=1
\usepackage{url,color,xspace}
\usepackage{algorithm,algorithmic}
\usepackage{amsmath}
\usepackage{multirow}
\usepackage{amssymb}
\usepackage{amsfonts}
\usepackage{graphicx}
 \usepackage[left=1in,right=1in,top=1in,bottom=1.25in]{geometry}

\usepackage{pifont}
\newcommand{\cmark}{\ding{51}}%
\newcommand{\xmark}{\ding{55}}%
 
\newcommand{\eps}{\varepsilon}
\newcommand{\methodname}{\textsc{ADAgIO}\xspace}

\newcommand{\beql}[1]{\begin{equation}\label{#1}}
\newcommand{\eeq}{\end{equation}}

\newtheorem{proof}{\smallskip\noindent\textit{\textbf{Proof}}\quad}

\newcommand{\field}[1]{\mathbb{#1}} 

\newcommand{\spara}[1]{\smallskip\noindent{\bf #1}}

\newtheorem{theorem}{Theorem}

\newtheorem{lemma}{Lemma}

\newcommand{\squishlist}{
 \begin{list}{$\bullet$}
  {  \setlength{\itemsep}{0pt}
     \setlength{\parsep}{3pt}
     \setlength{\topsep}{3pt}
     \setlength{\partopsep}{0pt}
     \setlength{\leftmargin}{2em}
     \setlength{\labelwidth}{1.5em}
     \setlength{\labelsep}{0.5em}
} }
\newcommand{\squishlisttight}{
 \begin{list}{$\bullet$}
  { \setlength{\itemsep}{0pt}
    \setlength{\parsep}{0pt}
    \setlength{\topsep}{0pt}
    \setlength{\partopsep}{0pt}
    \setlength{\leftmargin}{2em}
    \setlength{\labelwidth}{1.5em}
    \setlength{\labelsep}{0.5em}
} }

\newcommand{\squishdesc}{
 \begin{list}{}
  {  \setlength{\itemsep}{0pt}
     \setlength{\parsep}{3pt}
     \setlength{\topsep}{3pt}
     \setlength{\partopsep}{0pt}
     \setlength{\leftmargin}{1em}
     \setlength{\labelwidth}{1.5em}
     \setlength{\labelsep}{0.5em}
} }

\newcommand{\squishend}{
  \end{list}
}

\newcommand{\squishlistt}{
 \begin{list}{---}
  {  \setlength{\itemsep}{0pt}
     \setlength{\parsep}{3pt}
     \setlength{\topsep}{3pt}
     \setlength{\partopsep}{0pt}
     \setlength{\leftmargin}{2em}
     \setlength{\labelwidth}{1.5em}
     \setlength{\labelsep}{0.5em}
} }

\begin{document}
 
\title{ADAGIO: Fast Data-aware Near-Isometric Linear Embeddings}

\author{Jaros{\l}aw B{\l}asiok\thanks{Harvard University,  \texttt{jblasiok@g.harvard.edu}. Supported by NSF grant IIS-1447471.}
\and
Charalampos E. Tsourakakis\thanks{Boston University, Harvard University, \texttt{babis@seas.harvard.edu}. }
}
 
\maketitle

\begin{abstract}

Many important applications, including signal reconstruction, parameter estimation,  and signal processing in a compressed domain, rely on a low-dimensional representation of the dataset  that preserves {\em all} pairwise distances between the data points and leverages the inherent geometric structure that is typically present. 
Recently Hedge, Sankaranarayanan, Yin and Baraniuk \cite{hedge2015} proposed the first data-aware near-isometric linear embedding which achieves the best of both worlds. However, their method NuMax does not scale to large-scale datasets. 

Our main contribution is a simple, data-aware, near-isometric linear dimensionality reduction method which significantly outperforms a state-of-the-art method \cite{hedge2015} with respect to scalability while achieving high quality  near-isometries.  Furthermore, our method comes with strong worst-case theoretical guarantees that allow us to guarantee the quality of the obtained near-isometry. We verify experimentally the efficiency of our method on numerous real-world datasets, where we find that our method ($<$10 secs) is more than 3\,000$\times$ faster than the state-of-the-art method \cite{hedge2015}  ($>$9 hours) on  medium scale datasets with 60\,000 datapoints in 784 dimensions. Finally, we use our method as a preprocessing step 
to increase the computational efficiency of a classification application and for speeding up approximate nearest neighbor queries. 
\end{abstract}

\section{Introduction}
\label{sec:introduction}
Linear dimensionality reduction techniques lie at the core of high dimensional data analysis, due to their appealing 
geometric interpretations and computational properties. Among such techniques, principal component analysis (PCA) and its variants \cite{candes2011robust,d2007direct,scholkopf1997kernel,tipping1999probabilistic} 
play a key role in analyzing high dimensional datasets.  
However, it is well-known that PCA can significantly distort pairwise distances between sample data points \cite{achlioptas2003database}. This is an important drawback for machine learning and signal processing applications that rely on fairly accurate distances between points, such as reconstruction and parameter estimation \cite{hedge2015}.

By relaxing the notion of isometry into a {\em near-isometry},   the seminal Johnson-Lindenstrauss lemma \cite{johnson1984extensions} shows that a Euclidean space with much smaller number of dimensions  preserves all pairwise distances up to relative error $\epsilon$. One simply needs to project the original space onto a random subspace and scale accordingly to obtain a low-distortion embedding with high probability. Formally, the JL-lemma states the following.

\begin{theorem}[JL lemma {\cite[Lemma 1]{johnson1984extensions}}]
\label{thrm:jllemma}
For any $n$-point subset $X$ of Euclidean space and any $0<\eps<1/2$, there exists a map $f:X\rightarrow\ell_2^m$ with $m = O(\eps^{-2}\log n)$ such that
\begin{equation}
\forall x,y\in X,\ (1-\eps)\|x-y\|_2^2 \le \|f(x)-f(y)\|_2^2 \le (1+\eps)\|x-y\|_2^2 . 
\label{eqn:jl-guarantee}
\end{equation}
\end{theorem}

To quantify the notion of near-isometry, we introduce the notion of distortion: if $f$  is a linear map, the distortion for a pair of points $x,y$ is defined as 

\begin{equation}
\text{distortion}(x,y) =  | \frac{\|f(x)-f(y)\|_2}{\|x-y\|_2} - 1|. 
\label{eqn:distortiondfn}
\end{equation}

\begin{figure*}[t]
    \begin{center}
    \includegraphics[width=0.4\textwidth]{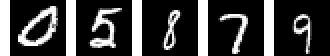}
\end{center}
    \caption{
        Example images from MNIST dataset. Each image consists of $28\times 28$ pixels, and is considered as a point in a 784-dimensional space.
        \label{fig:mnist-datapoints}
    }
\end{figure*}
\begin{figure*}[t]
\centering
\begin{tabular}{@{}c@{}@{\ }c@{}@{\ }c@{}}
 \includegraphics[width=0.23\textwidth]{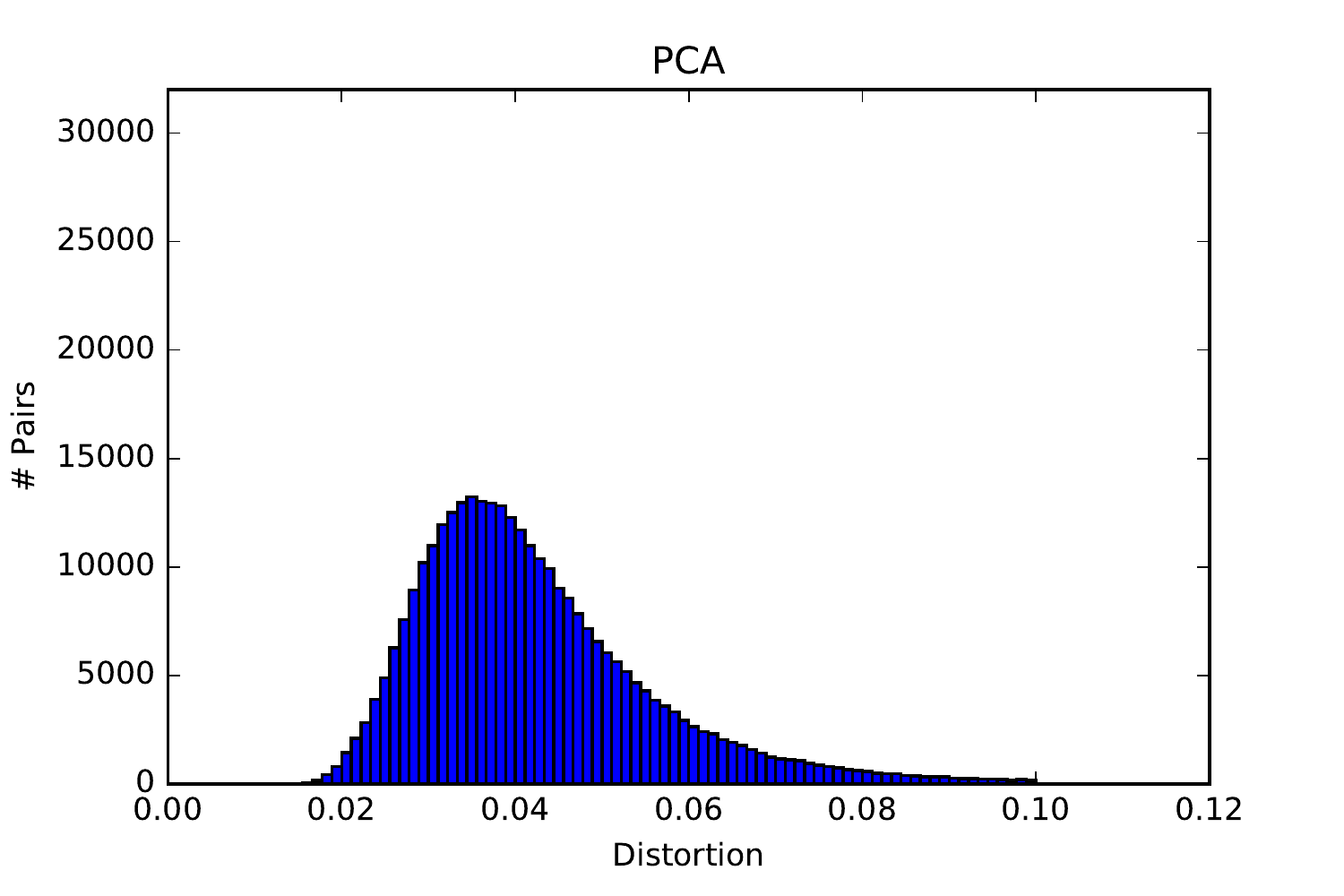}  \hspace{6mm} & \includegraphics[width=0.23\textwidth]{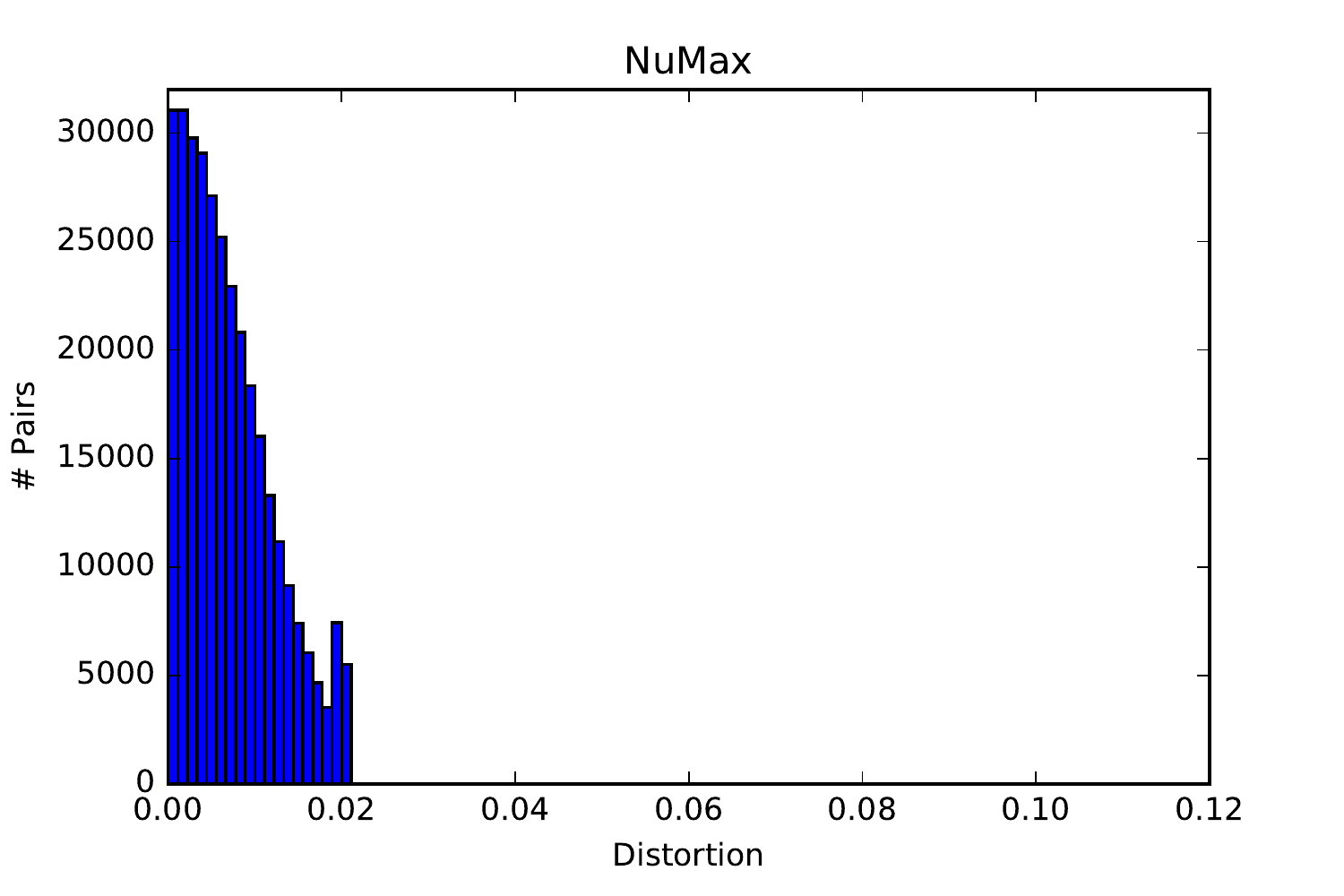} \hspace{6mm}  & \includegraphics[width=0.23\textwidth]{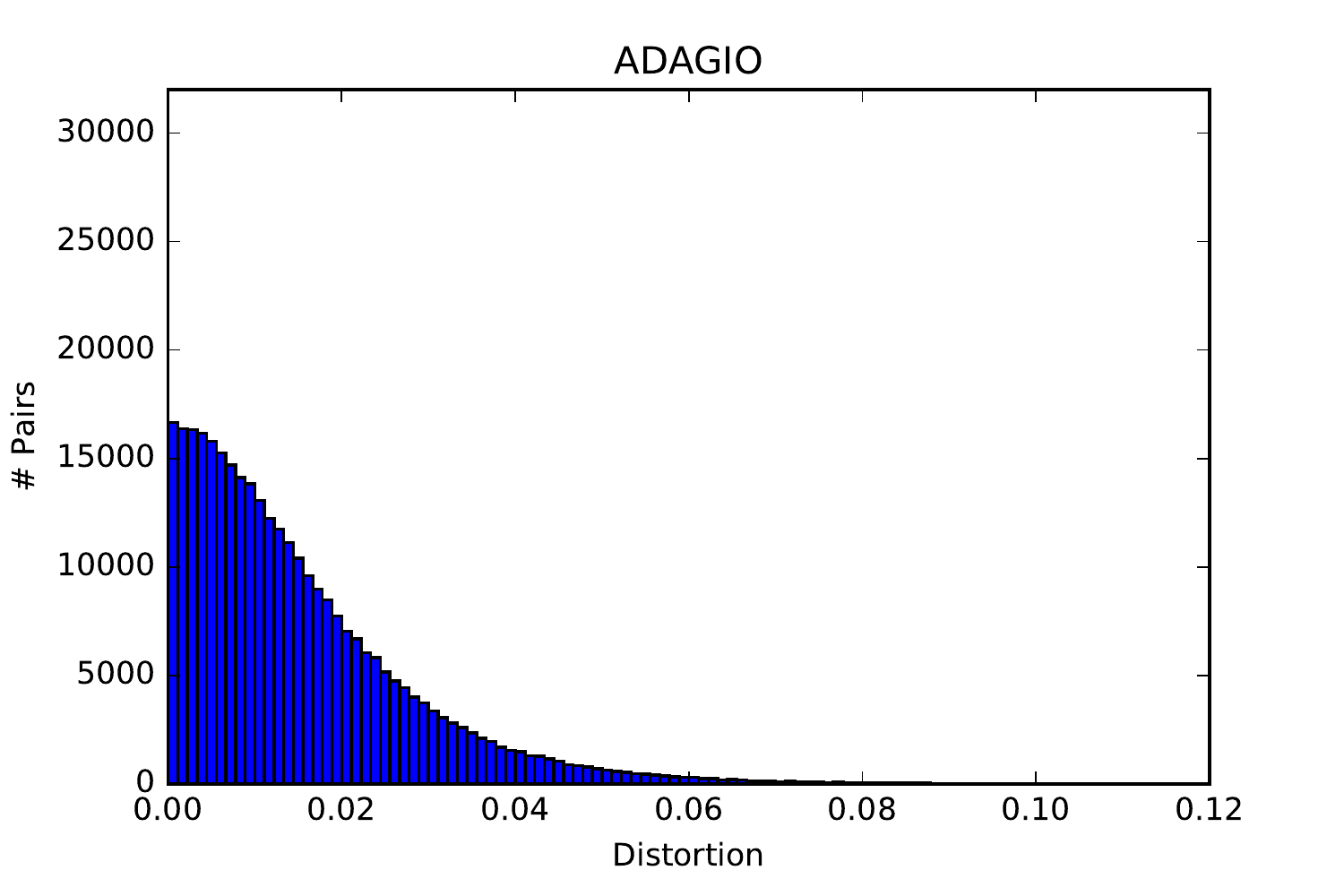} \\
(a) & (b) &(c) 
\end{tabular}
\caption{\label{fig:mnistintro} Distribution of pairwise distortions for (a) PCA, (b) Numax, and (c) \methodname. Respective run times in seconds are  5.2, 1\,105, 5.6 seconds respectively.}
\vspace{-4mm}
\end{figure*}

\noindent The maximum distortion between any pair of points in the cloud is defined as the {\em distortion} of the embedding $f$, and quantifies how close  the embedding is to being a near-isometry.

\begin{equation}
\text{distortion}(f) =  \max_{x,y} | \frac{\|f(x)-f(y)\|_2}{\|x-y\|_2} - 1|. 
\label{eqn:distortionemb}
\end{equation}

\noindent  The JL-lemma equivalently states that there exists a linear embedding to a low dimensional subspace whose distortion is at most $\eps$.  However, random projections are data oblivious and hence do not leverage any special geometric structure of the dataset, which frequently is present. For this reason, recent work has focused on  developing linear dimensionality reduction techniques that are both {\em data-aware} and produce {\em near-isometric} embeddings, cf.  \cite{bah2013energy,	grant2013nearly,hedge2015,kyrillidis2014approximate,luo2016practical}. It is worth mentioning that non-linear techniques have also been developed, see, e.g., \cite{kulis2007fast}.

In their recent work Hedge, Sankaranarayanan, Yin and Baraniuk \cite{hedge2015} provide a non-convex formulation and a convex relaxation. On the one hand their method NuMax is data-aware and outputs high-quality near-isometries, but on the other hand  does not scale well: on a dataset with 1000 data points in $\field{R}^{256}$,  NuMax requires more than half an hour on a usual laptop to perform dimensionality reduction. For even higher dimensional datasets or datasets with more data points, NuMax fails frequently to produce any results in a reasonable amount of time. Furthermore, from a theory perspective it  is an
interesting open problem to prove any kind of theoretical guarantees for NuMax, as at the moment its performance is not well understood.

\spara{Contributions.}  Our contributions are summarized as follows.

\squishlist
\item We propose a novel randomized method for constructing efficiently linear, data-aware, near-isometric embeddings of  a high-dimensional dataset.  Our method combines the best of the PCA-world (data-awareness) and the JL-world (near-isometry) in an intuitive way:  combine PCA and random projections appropriately.   Also, our method comes with strong theoretical guarantees in contrast to prior work.
\smallskip 
\item We verify the effectiveness of our proposed method on numerous real-world high-dimensional datasets. We observe that our method outperforms the state-of-the-art method due to Hedge et al. \cite{hedge2015} significantly in terms of run times, while achieving high-quality near-isometries. Also, our method in constrast to NuMax comes is amenable to distributed implementations.
\smallskip 
\item We use our method as a preprocessing step for an approximate nearest neighbor (ANN) application. Our findings indicate that our method can be used to improve the efficiency of ANN algorithms while achieving high accuracy.  Our findings are similar for a classification application using k-nearest neighbor graphs. We observe that \methodname achieves high accuracy and significant speedups  compared to NuMax \cite{hedge2015}, ranging from few hundred times up to hundred thousand times (specifically  424\,780$\times$, \methodname runs in 0.01 seconds while NuMax requires 4\,247.8 seconds).
\squishend

\smallskip

Figure~\ref{fig:mnistintro} offers a quick illustration of our contribution on the digit MNIST dataset. For the exact details of the experimental setup, confer Section~\ref{sec:exp}. Figure~\ref{fig:mnist-datapoints} shows few data points: each data point is an image of a handwritten digit with 28$\times$28 pixels. Each image is converted to a data point in 784 dimensions. The dataset is publicly available \cite{mnist}. Figure~\ref{fig:mnistintro}(a),(b),and (c) show the distribution of pairwise distortions for PCA, Numax and our proposed method \methodname respectively. For the dataset specifics and the exact experimental setup see Section~\ref{sec:exp}. Note  that zero distortion corresponds to a  perfect isometry. We observe that both NuMax and \methodname improve significantly PCA's performance with respect to achieving a near-isometry. Also, both NuMax and \methodname leverage the geometric data structure.   However, in terms of runtimes, \methodname runs in 5.6 seconds whereas Numax in roughly 18.4 minutes.
  
\spara{Roadmap.}  Section~\ref{sec:related}  reviews briefly  related work. Section~\ref{sec:proposed} presents our mathematical and algorithmic contributions, and Section~\ref{sec:exp} evaluates experimentally our framework. Finally, Section~\ref{sec:concl} concludes the paper.

\section{Related Work}
\label{sec:related}
Linear dimensionality reduction methods are a cornerstone for analyzing high dimensional datasets. Given a cloud $\mathcal{T}$ of $n$ $d$-dimensional points $X^T = [x_1|\ldots|x_n]$ ($X \in \field{R}^{n \times d}$), where $x_i \in \mathcal{T}$, and a target dimension $r<d$, the goal is to produce a matrix $\Pi \in \field{R}^{r \times d}$  that transforms {\em linearly} $X$ into $Y = \Pi X \in \field{R}^{r\times n}$ and optimizes some objective function.

\spara{PCA.} Principal Component Analysis (PCA) is a Swiss Army knife for data analysis. It minimizes the reconstruction error under the $l_2$ norm $f_X(M) = || X - MM^TX||_F^2$. It is well known, that once the data is zero-centered (i.e. $\sum\limits_{x \in \mathcal{T}} x = 0$), PCA yields an orthogonal projection matrix that minimizes the following quantity

    \begin{equation}
        \sum_{x,y \in T} \Big( \|x - y\|^2 - \| \Pi x - \Pi y\|^2 \Big)
        \label{eq:avg-distortion}
    \end{equation}

\noindent  over all orthogonal projections of rank $k$. PCA is obtained by computing the singular value decomposition (SVD) of matrix $X = U \Sigma V^T$ and then by projecting the columns of $X$ on the subspace spanned by the $r$  leftmost columns of $U$.   As it can be seen by the Equation~\eqref{eq:avg-distortion}, PCA aims to minimize a quantity related to average distortion over all pairs of points rather than the maximum distortion. In practice PCA is not near-isometric:  typically there exists a non-trivial fraction of pairs of points whose distance can be very distorted.

\spara{Random projections.}  Principal component analysis (PCA) is a classical unsupervised analysis method, and the predominant linear dimensionality reduction technique. While PCA is an important algorithmic tool for analyzing high-dimensional datasets, many important signal processing applications such as reconstruction and parameter estimation  rely on preserving well {\em all} pairwise distances \cite{hedge2015}. Random projections provide a simple way to obtain near-isometric linear embeddings of clouds of points \cite{achlioptas2003database,dasgupta2003elementary,frankl1988johnson,johnson1984extensions,larsen2014johnson},
 see also  \cite{bingham2001random,boufounos2015representation,indyk1998approximate} for some more  practical aspects. Besides the existential result stated as Theorem~\ref{thrm:jllemma}, an efficient construction of the embedding is possible. 
\medskip

 \begin{lemma}[Constructive Johnson-Lindenstrauss \cite{achlioptas2003database}]
        \label{lem:jl}
        For any finite set of points $\mathcal{T} \subset \field{R}^d$, let $\Pi \in \field{R}^{k \times d}$ be a matrix with uniformly and independently chosen random entries  $\pi_{ij} \in \{-\frac{1}{\sqrt{k}}, \frac{1}{\sqrt{k}}\}$. Moreover let $k = \Omega\Big(\frac{1}{\delta^2} \log\big( \frac{|T|}{\gamma}\big) \Big)$. Then with probability at least $1-\gamma$, it holds that
        \begin{equation}
            \forall_{x, y \in \mathcal{T} } (1-\delta)\|x - y\| \leq \| \Pi x - \Pi y\| \leq (1+\delta)\|x - y\|
        \end{equation}
    \end{lemma}

\spara{NuMax \cite{hedge2015}.}   Surprisingly, until recently  little was known in terms of designing  data-aware, near-isometric linear dimensionality reduction techniques. Such methods, like random projections preserve well pairwise distances, and at the same time   leverage the geometric structure --typically inherent-- of the dataset.   In their work Hedge, Sankaranarayanan, Yin and Baraniuk \cite{hedge2015} approached this question --- they stated it as an non-convex optimization problem, and proposed a convex relaxation, which they called NuMAX. Despite the fine-tuned design of their algorithm, it still does not scale to large-scale datasets.

\section{Proposed Method}
\label{sec:proposed}
Our proposed method \methodname (fAst Data Aware IsOmetry) is shown in Algorithm~\ref{alg1}.  The algorithm takes as input a cloud of points $T$ and an orthonormal matrix $P \in \field{R}^{s \times d}$, i.e., $PP^T=I$. We are particularly interested in the setting where $s \ll d$. Note that $P^T P$ is the projection matrix on the subspace spanned by the rows of $P$. It is useful to think as $P^TP$ as the PCA projection matrix,namely $P$ contains the top $s$ left  singular vectors of the data matrix.    In general any other type of projection matrix can be used instead, for instance \cite{AbdullahAKK14,candes2011robust,tipping1999probabilistic}. The algorithm generates a random projection JL-matrix $S$ and maps linear each $w \in T$ to 
$(Pw, S(w - P^TPw))$. This surprisingly simple algorithm provably 
achieves fast, data-aware, and near-isometric dimensionality reduction.

\begin{algorithm}                      
\caption{\methodname}          
\label{alg1}                           
\begin{algorithmic}                    
    \REQUIRE Cloud of points $T \subset \field{R}^d$, orthonormal matrix $P^{s \times d}$ with distortion $\delta$, distortion reduction parameter $\epsilon <1$
     \STATE $k \leftarrow \Theta\big(\frac{1}{\epsilon^2} \log(|T|)\big) $
    \STATE Let $S \in \field{R}^{k \times d}$ be a random matrix with uniformly and independently chosen random entries  $s_{ij} \in \{-\frac{1}{\sqrt{k}}, \frac{1}{\sqrt{k}}\}$ 
    \FOR{each $w \in T$}
     \STATE Output $f(w) \leftarrow (Pw, S(w - P^TPw))$
     \ENDFOR
\end{algorithmic}
\end{algorithm}

\begin{figure*}[!ht]
\centering 
\begin{tabular}{@{}c@{}@{\ }c@{}@{\ }c@{}} \includegraphics[width=0.28\textwidth]{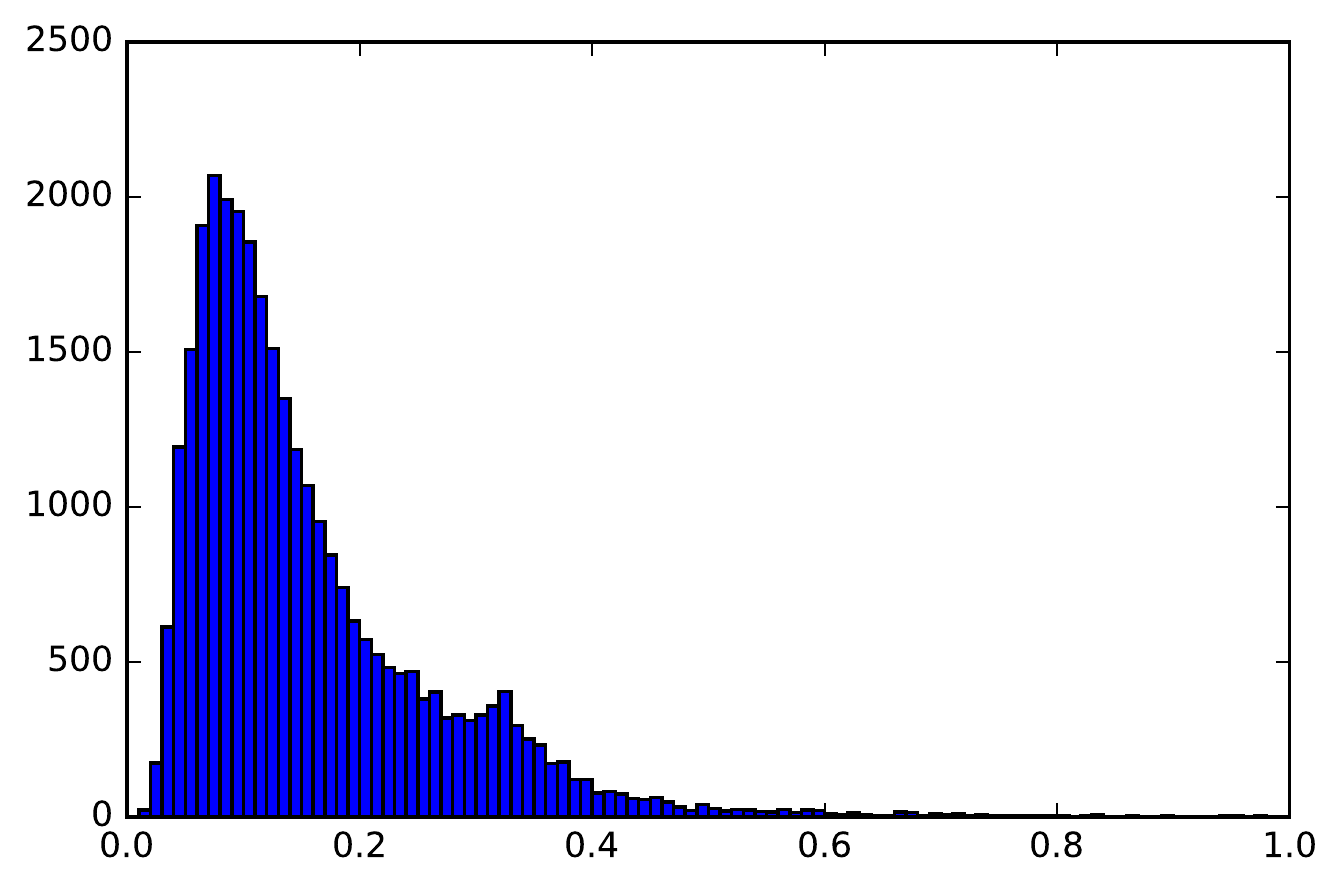} \hspace{6mm} &
\includegraphics[width=0.28\textwidth]{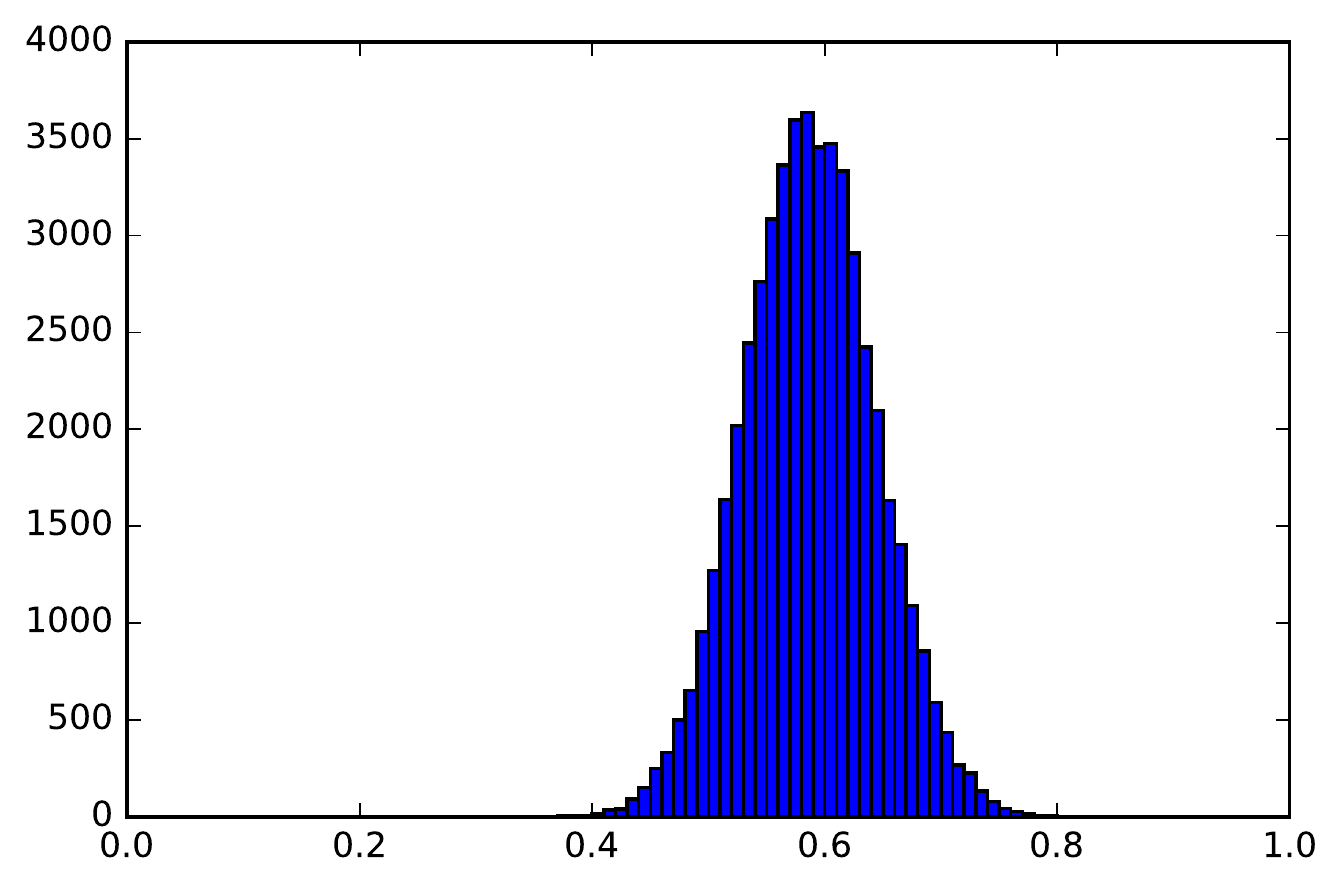}  \hspace{6mm}  & \includegraphics[width=0.28\textwidth]{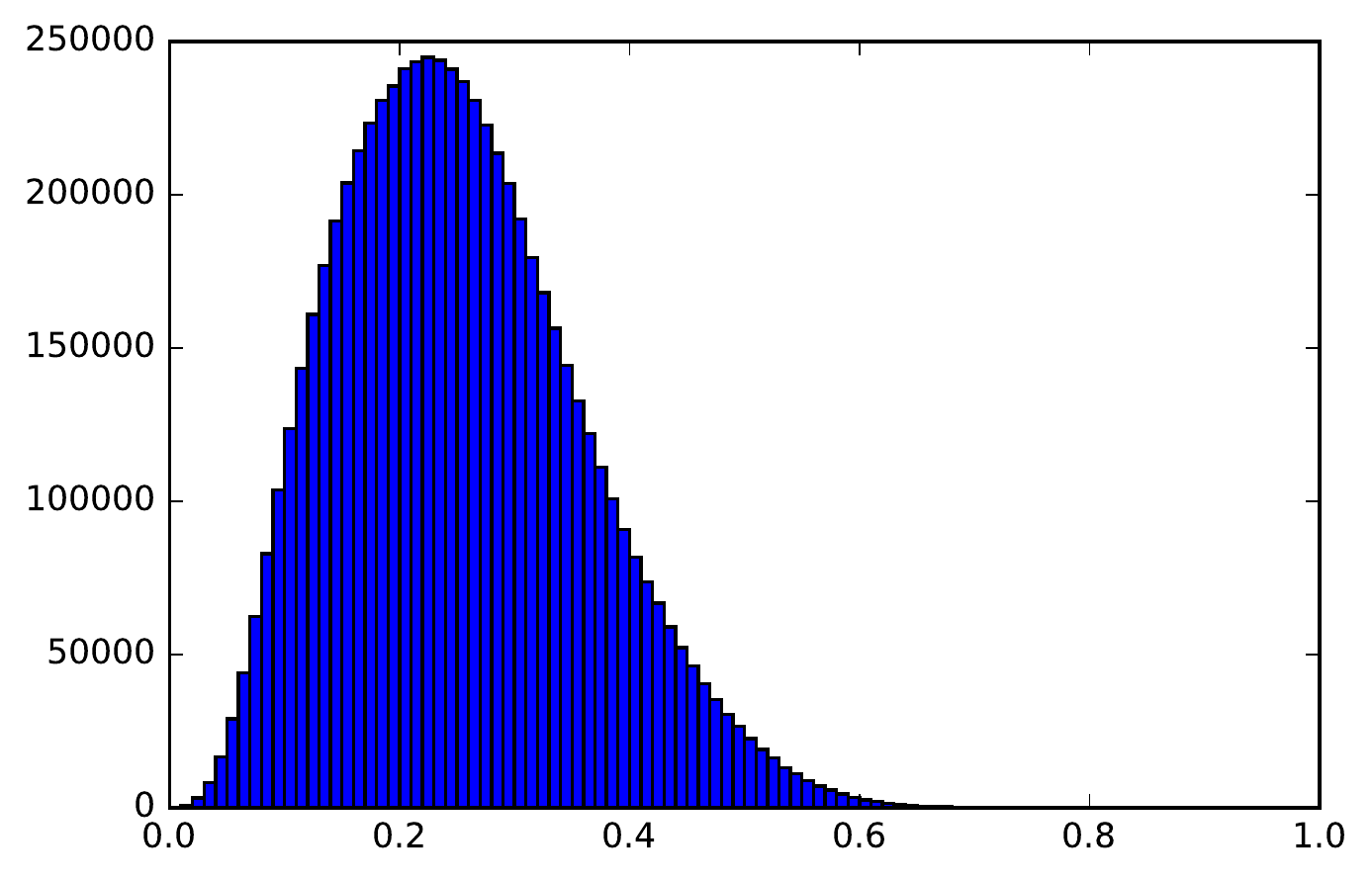} \\
(a) & (b) & (c) \\
\includegraphics[width=0.28\textwidth]{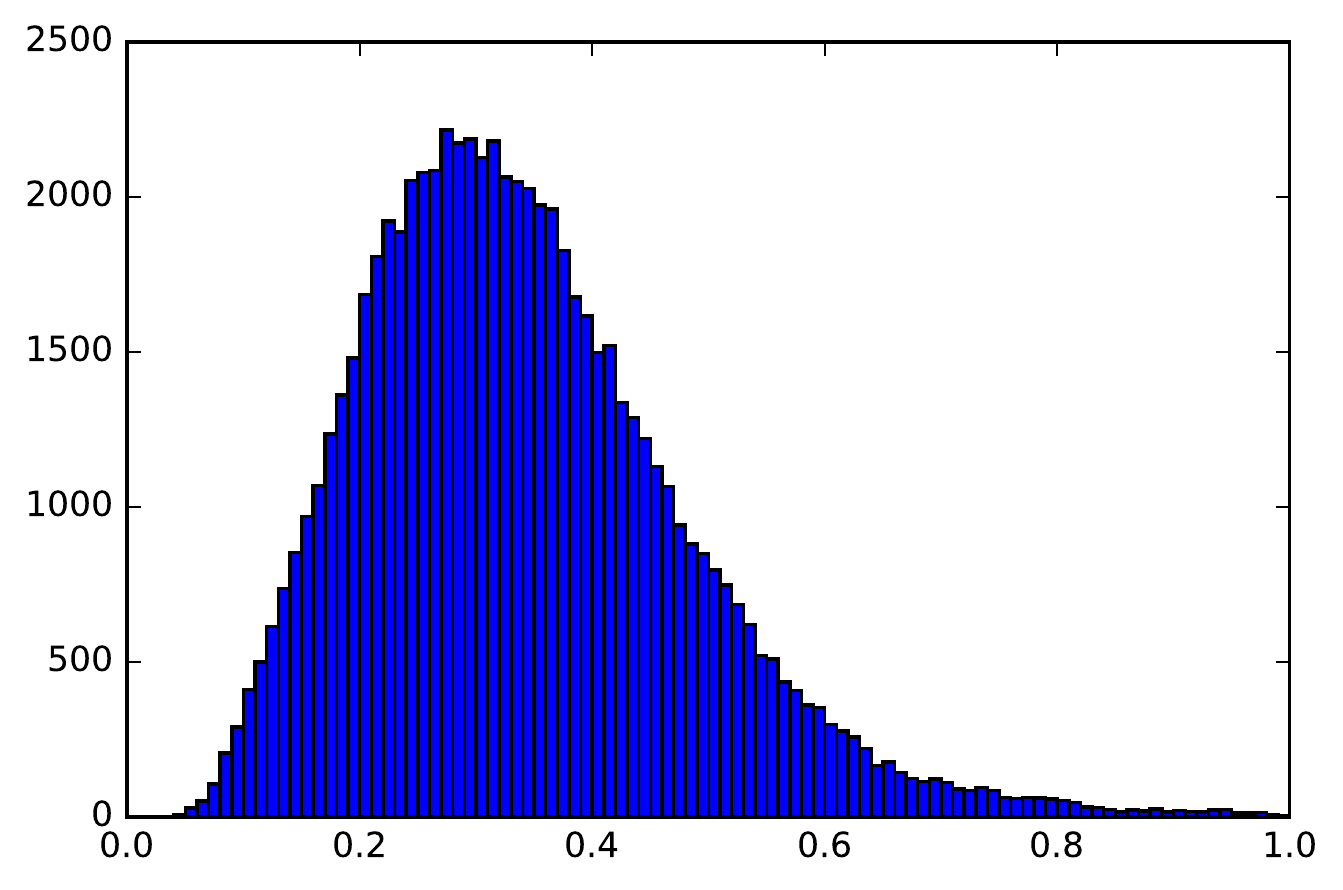} \hspace{6mm}  &
\includegraphics[width=0.28\textwidth]{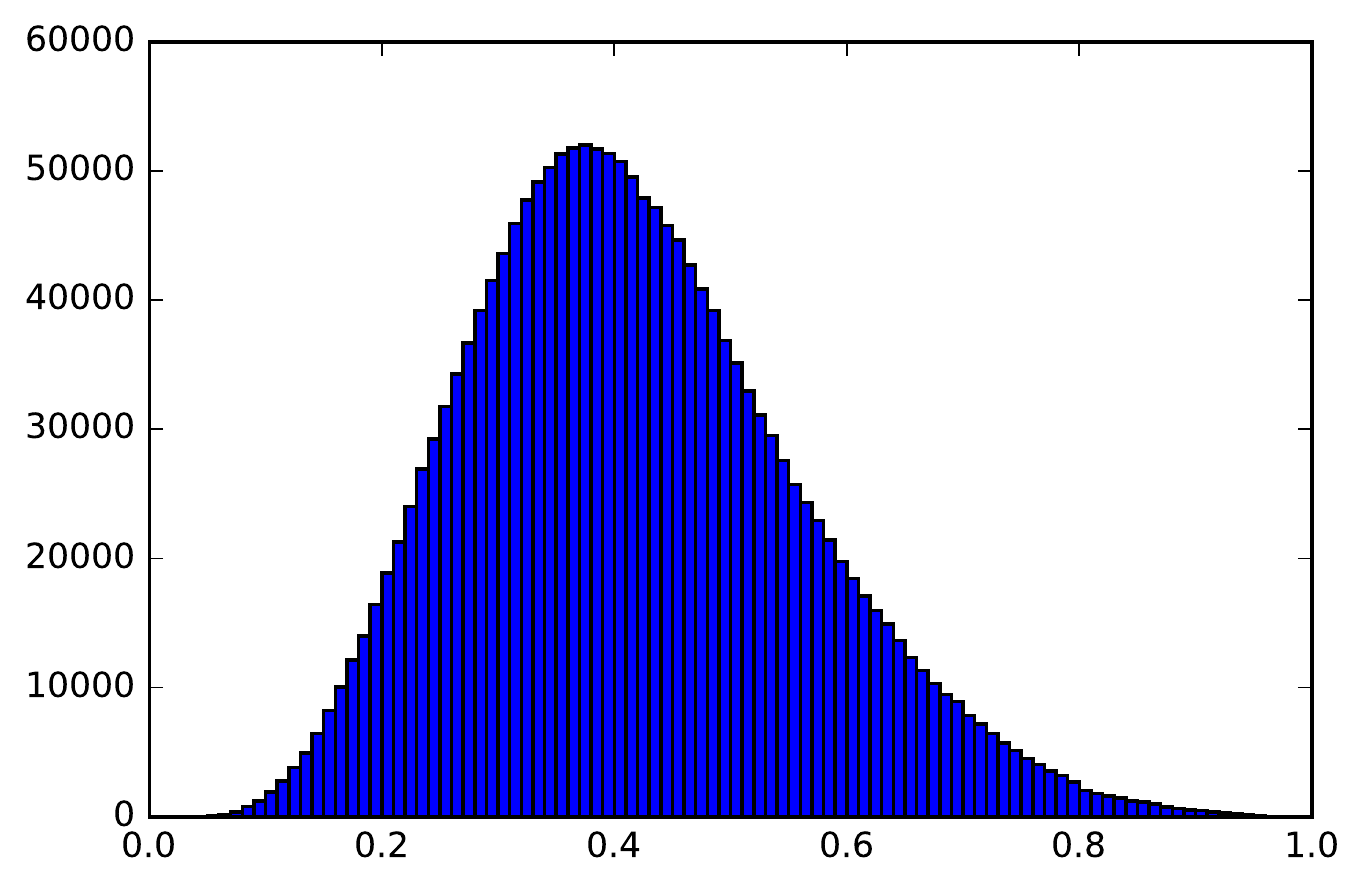} \hspace{6mm}  & \includegraphics[width=0.28\textwidth]{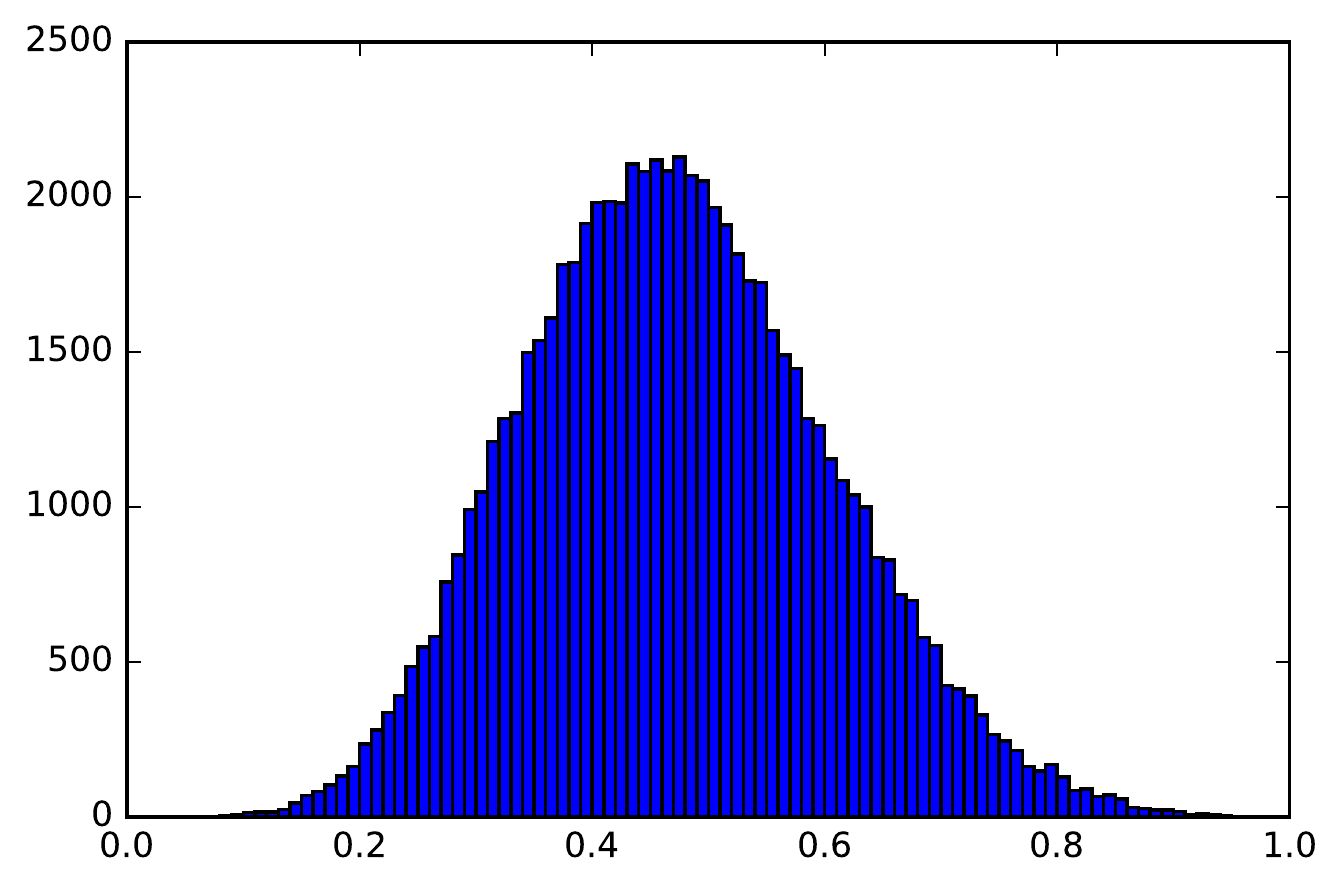} \\
(d) & (e) & (f) \\
\includegraphics[width=0.28\textwidth]{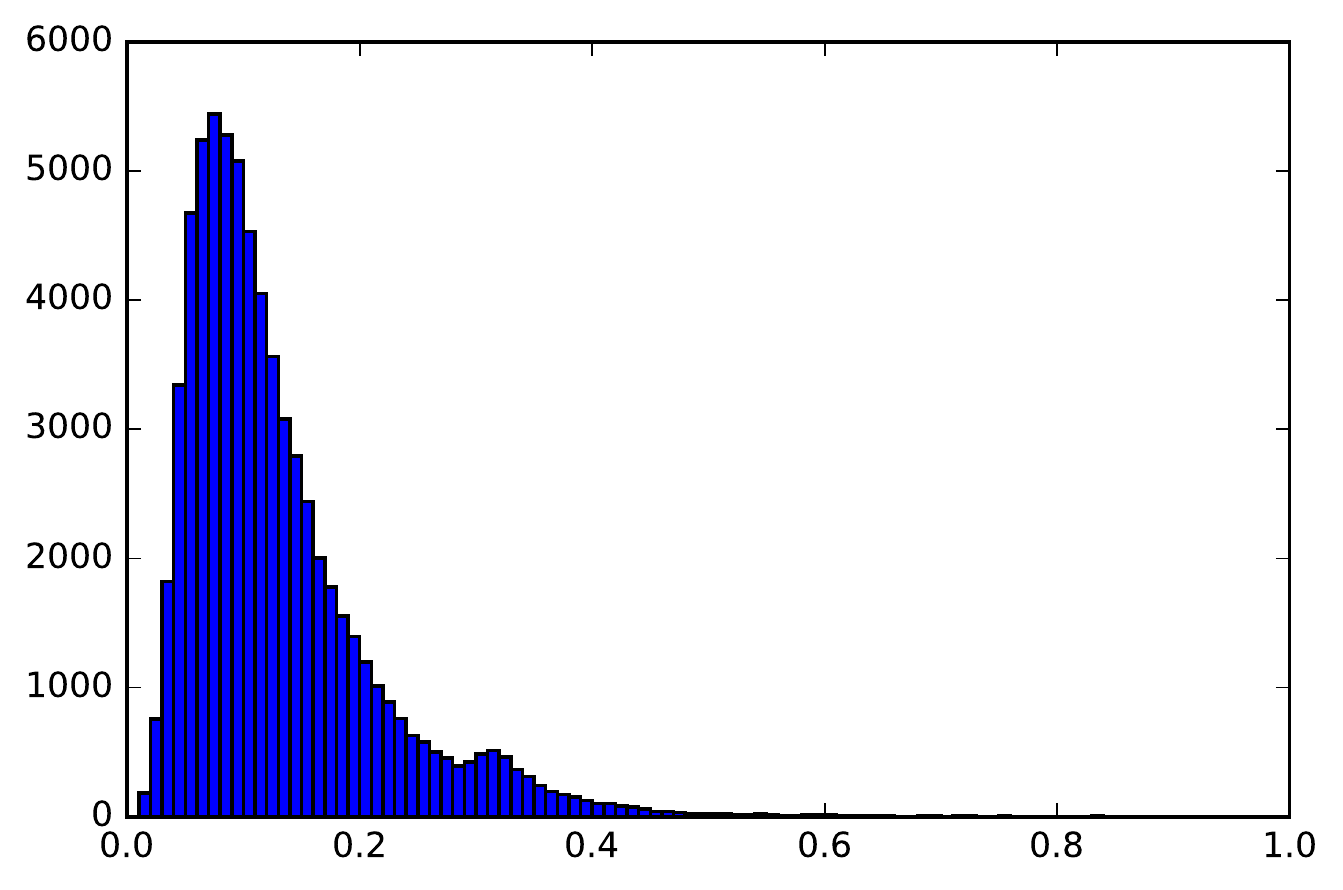} \hspace{6mm}  &
\includegraphics[width=0.28\textwidth]{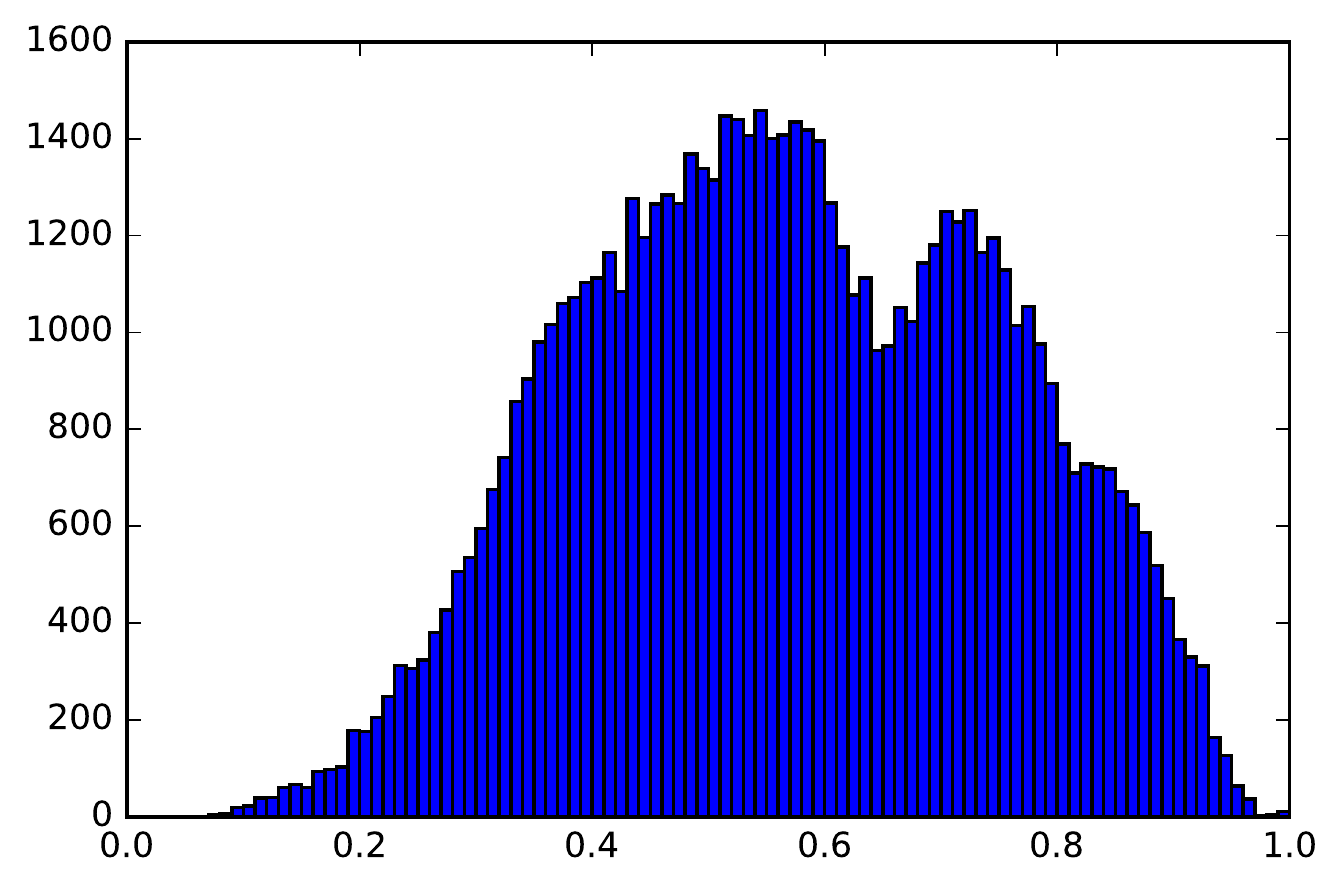} \hspace{6mm}  & \includegraphics[width=0.28\textwidth]{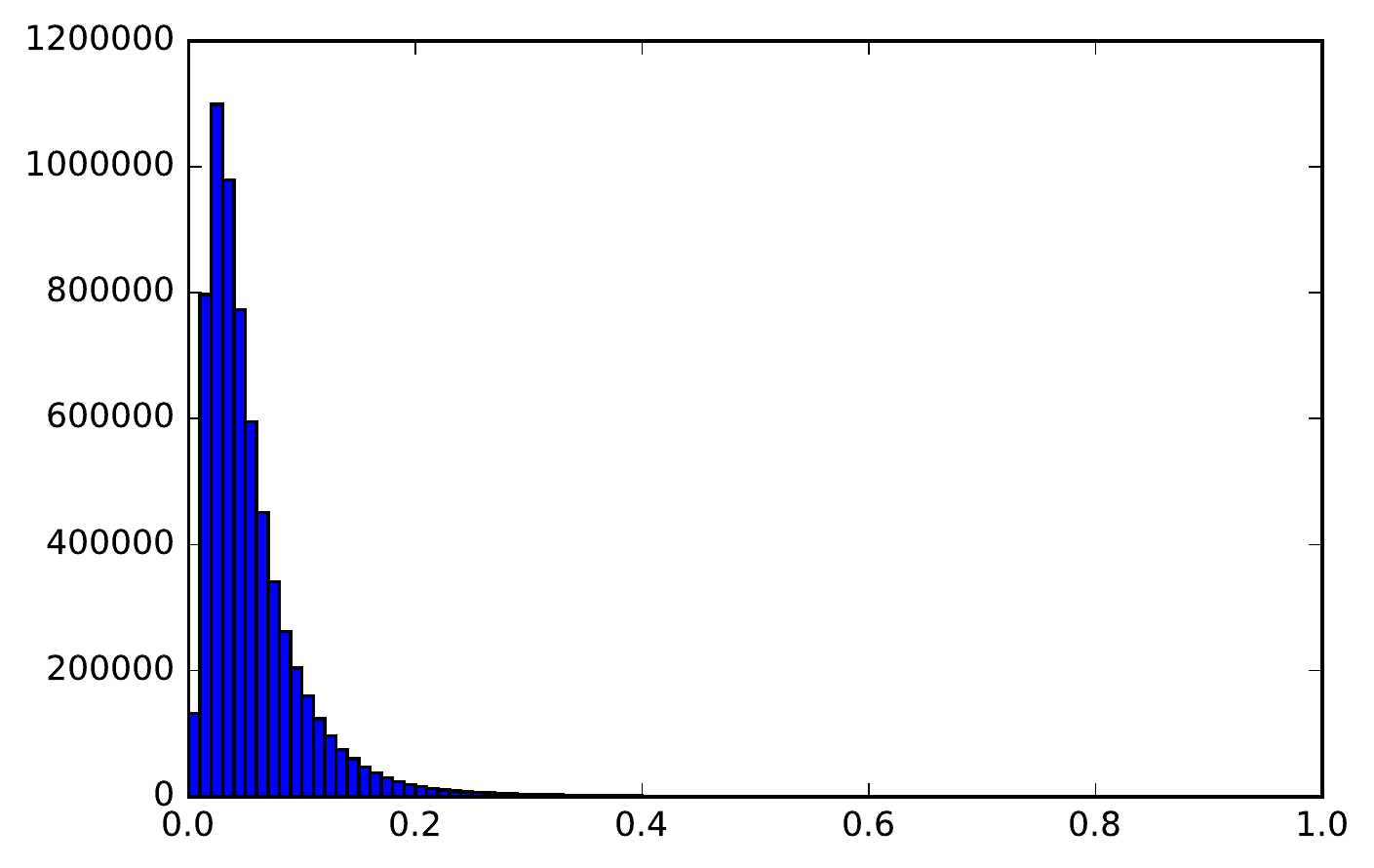} \\
(g) & (h) & (i) \\
\end{tabular}
\caption{\label{fig:pcadistortion} Distribution of distortion for PCA
$\left\{\left|\frac{\|\Pi x - \Pi y\|}{\|x - y\|} -1 \right| \right\}$ for all pairs of points $x,y \in T$. Here, the target dimension is set to 20 for  
(a) Computers, (b)  Earthquakes, (c) FordB,
(d) LargeKitchenAppliances  (e) Phoneme (f) RefrigerationDevices, (g)  ScreenType, (h) SmallKitchenAppliances, (i)   UWaveGesture. } 
\vspace{-4mm}
\end{figure*}

We prove a general theoretical result stated as Theorem~\ref{thm:main-theorem}. 
Then, we perform an extensive empirical evaluation of PCA with respect to how it distorts pairwise distances. Our findings show that  PCA's projection matrix  can be used as input $P$ to Algorithm~\ref{alg1}. 

\spara{Padding Johnson-Lindenstrauss dimensions.} Intuitively, the next theorem states that given a linear dimensionality reduction technique, we can append sufficiently many JL-dimensions to produce an embedding with smaller distortion.

\begin{theorem}
Let $\mathcal{T} \subset \field{R}^d$ be a finite set of points, and let $P \in \field{R}^{s \times d}$ be an orthonormal matrix with distortion $\delta$. Then, it is possible to reduce the distortion by a multiplicative factor $\epsilon<1$. Specifically, we can efficiently construct a matrix $\Pi \in \field{R}^{r \times n}$ such that with probability $0.99$, $\Pi$ has distortion $\delta\epsilon$. Here, the target dimension $r$ is equal to $s + O(\frac{\log |T|}{\epsilon^2})$. 
\label{thm:main-theorem}
\end{theorem}

\begin{proof}
Consider the normalized secant set $\mathcal{S}(T) = \{\frac{x - y}{\|x - y\|} : x, y \in T\}$. Observe, that since $P$ has distortion at most $\delta$ with respect to $T$ 
\begin{equation}
    \forall_{v \in \mathcal{S}(T) } \|P v\| \geq (1-\delta).
\end{equation}

Let $Q = I-P^TP$ be the projection matrix on the orthogonal subspace to  the subspace spanned by the rows of $P$. 
Moreover, let $S \in \field{R}^{s\times d}$ be a random projection matrix, with $s =\Theta(\frac{\log |T|}{\epsilon^2})$ rows. By 
 invoking the JL lemma~\ref{lem:jl} with $\gamma = 0.01$ we obtain that with probability $0.99$ the following holds 
\begin{equation}
    \forall_{v \in\mathcal{S}(T) } (1-\beta/8) \|Qv\| \leq \|SQv\| \leq (1+\beta/8) \|Qv\|. 
    \label{}
\end{equation}

\begin{figure*}[t]
    \begin{center}
    \includegraphics[width=0.4\textwidth]{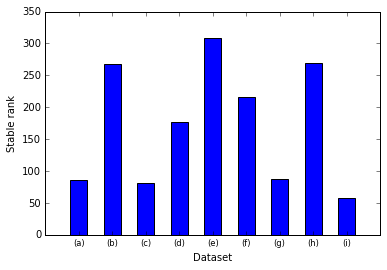}
\end{center}
    \caption{Stable rank $\frac{(\sum \sigma_i)^2}{\sum \sigma_i^2}$ for (a) Computers , (b)  Earthquakes, (c) FordB,
(d) LargeKitchenAppliances  (e) Phoneme (f) RefrigerationDevices, (g)  ScreenType, (h) SmallKitchenAppliances, (i)   UWaveGesture.
        \label{fig:stable}
    }
\end{figure*}

\noindent Let us condition on this event. Now, consider the map
\begin{equation}
    \Pi : w \mapsto (Pw, SQw) \in \field{R}^{d + s}
\end{equation}
    We claim that this map has distortion $\delta\epsilon$ with respect to $T$. It is enough to prove that for all $v \in \mathcal{S}(T)$, we have $1 -  \delta \epsilon \leq \|\Pi v\| \leq 1 + \delta \epsilon $. Fix any $v\in \mathcal{S}(T)$, and let $\zeta := \|Q v\|^2$. By Pythagoras Theorem, we have $1 = \|v\|^2 = \|Pv\|^2 + \|Q v\|^2 = \|P v\|^2 +\zeta$, by reordering $\|Pv\|^2 = 1 -\zeta$.  The assumption that $P$ has distortion at most $\delta$ with respect to $T$, yields that $\sqrt{1 - \zeta} = \|Pv\| \geq (1-\delta)$, and therefore 
    
    $$\zeta   \leq 2\delta.$$

On the other hand, we have
\begin{align*}
    \| \Pi v\|^2 & = \|Pv\|^2 + \|S Q w\|^2  \leq (1 - \zeta) + (1 + \frac{\epsilon}{8})^2 \|Q w\|^2 \\
    & \leq (1 -  \zeta) + (1 + \frac{\epsilon}{2})  \zeta  = 1 + \frac{\epsilon \zeta}{2} \\
    & \leq 1 + \delta \epsilon
\end{align*}

And similarly
\begin{align*}
    \|\Pi v\|^2 & = \|Pv\|^2 + \|S Q w\|^2 
  \geq (1 - \zeta) + (1 - \frac{\epsilon}{8})^2 \|Q w\|^2 \\
    & \geq (1 - \zeta) + (1 - \frac{\epsilon}{2})\zeta  = 1 - \frac{\beta\zeta}{2} \\
    & \geq 1 - \delta \epsilon
\end{align*}

And finally $\sqrt{1 - \delta\epsilon} \leq \|\Pi v\| \leq \sqrt{1 +  \delta\epsilon}$, which implies the desired result. QED
\end{proof}

\noindent A direct corollary of our result is that if the dataset is such that a low-rank linear dimensionality reduction method yields  distortion  $\sqrt{\delta}$, we can significantly reduce the distortion all the way to $\delta$, by padding $\frac{\log |T|}{\delta}$ JL dimensions. For comparison, with Johnson-Lindenstrauss the necessary target dimension would be of order $O\big( \frac{\log |T|}{\delta^2}\big)$. Even more importantly \methodname leverages the geometric structure which is typically inherent to the dataset.  

\spara{PCA and distance distortion.} By the triangle inequality and standard SVD properties we obtain 
 
\begin{align*}
|x-y| \leq |x-f(x)| +|y-f(y)| + |f(x)-f(y)| & \\ 
\leq 2\sigma_{k+1} +  |f(x)-f(y)| \rightarrow & \\ 
|x-y| - |f(x)-f(y)| \leq 2\sigma_{k+1}. &
\end{align*}

\noindent Here $\sigma_1 \geq \sigma_2 \geq \ldots$ are the  singular values of the data matrix ordered in non-increasing order.  Also, by Pythagoras's theorem $|x-y| \geq |f(x)-f(y)|$.  Combining these facts we obtain that the distortion of the pairwise distance between $x,y \in T$ satisfies 
 
\begin{align*}
| \frac{|f(x)-f(y)|}{|x-y|}-1\big| &\in [0, \min( \frac{2\sigma_{k+1}}{|x-y|}, 1)]
\end{align*}

\begin{figure*}[!ht]
\centering
\begin{tabular}{@{}c@{}@{\ }c@{}@{\ }c@{}} \includegraphics[width=0.28\textwidth]{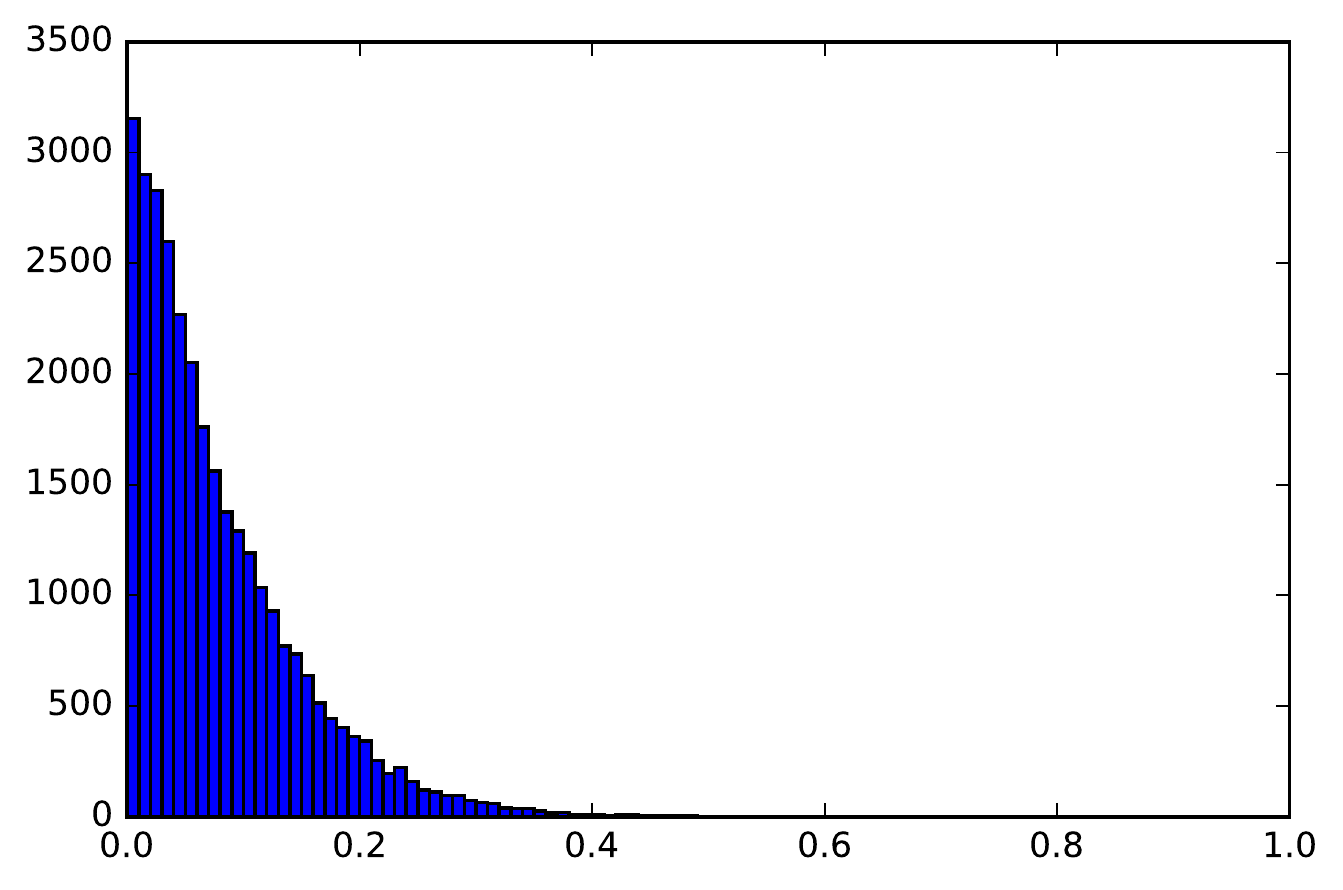} \hspace{6mm} &
\includegraphics[width=0.28\textwidth]{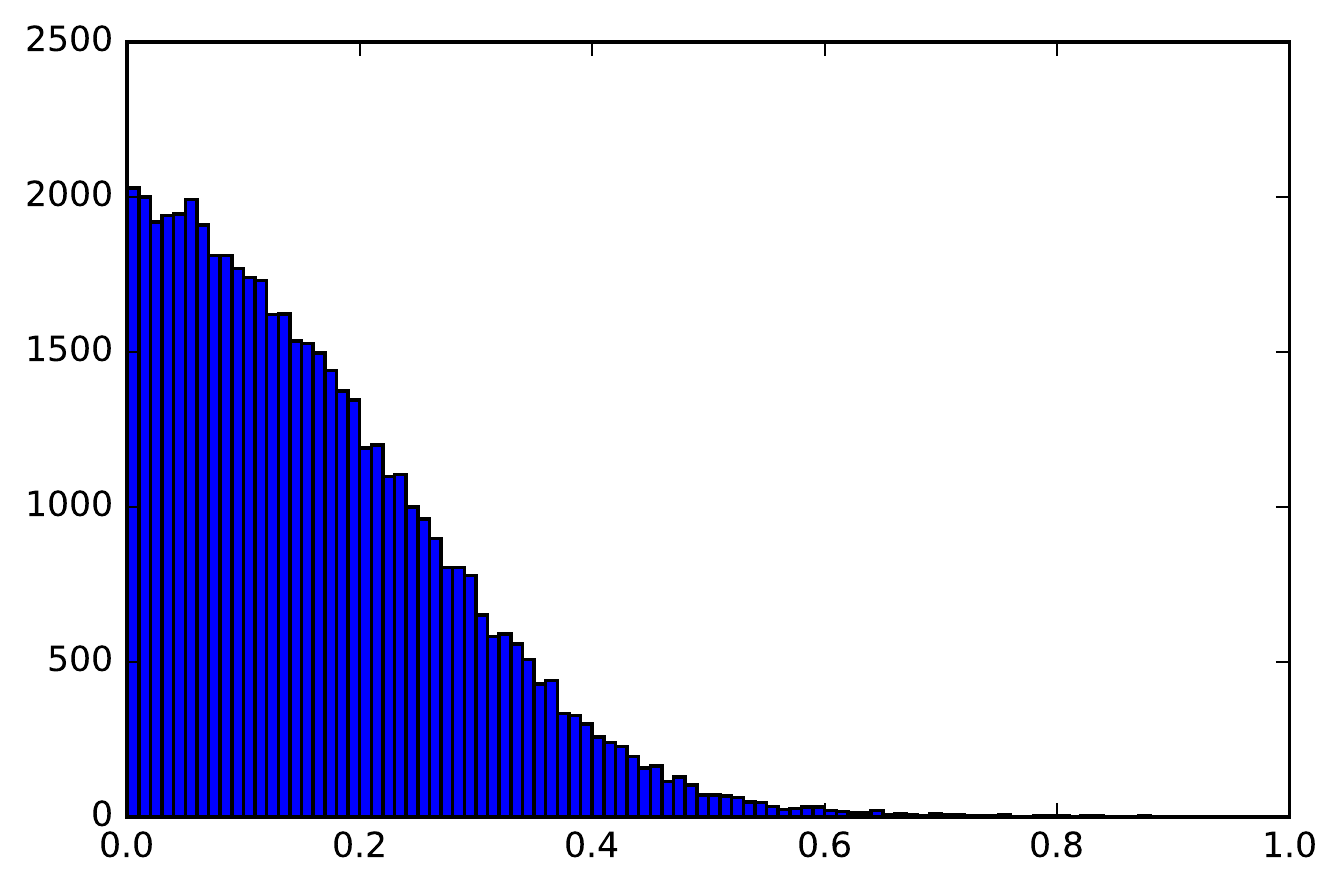}  \hspace{6mm}  & \includegraphics[width=0.28\textwidth]{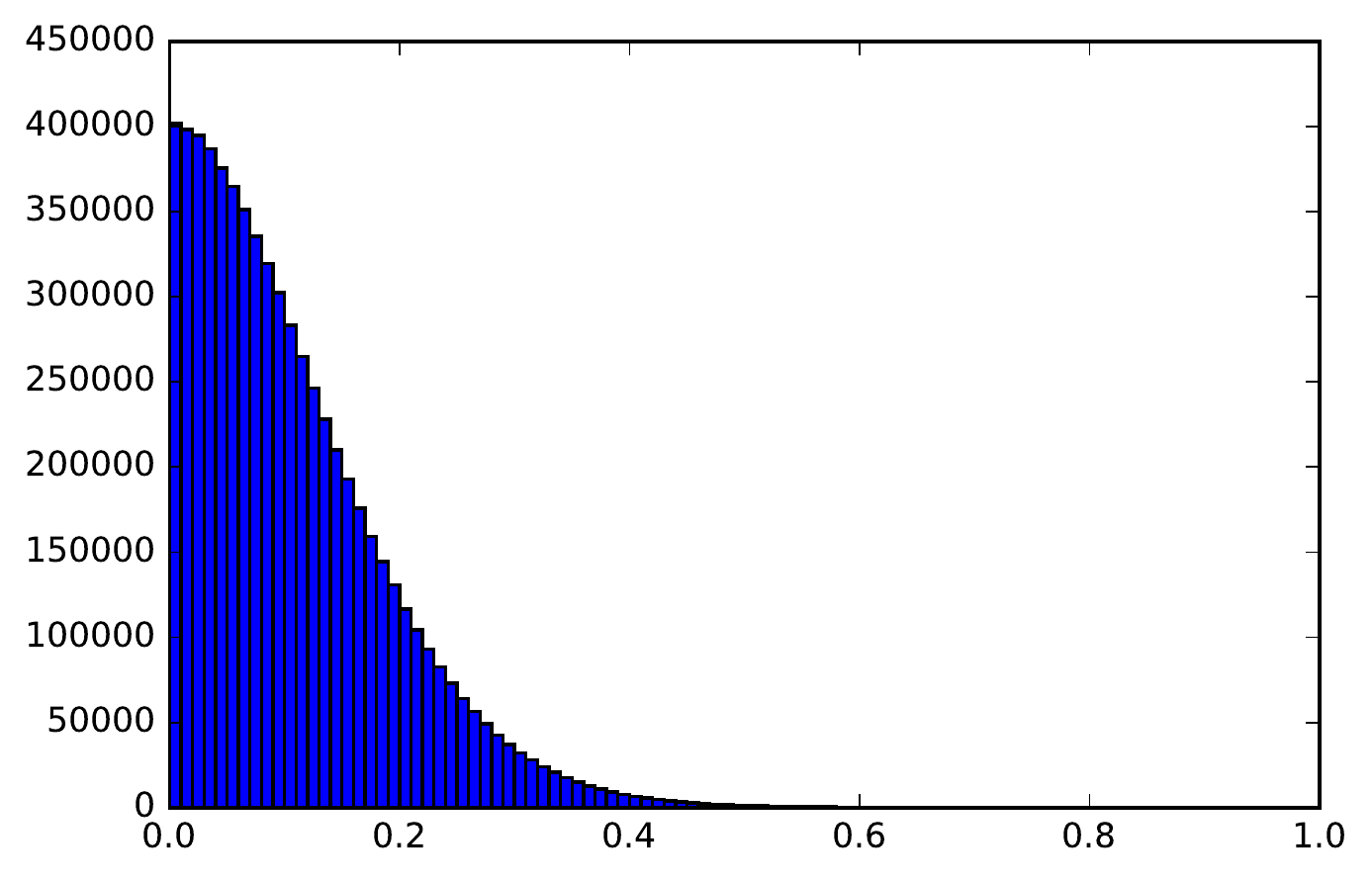} \\
(a) & (b) & (c) \\
\includegraphics[width=0.28\textwidth]{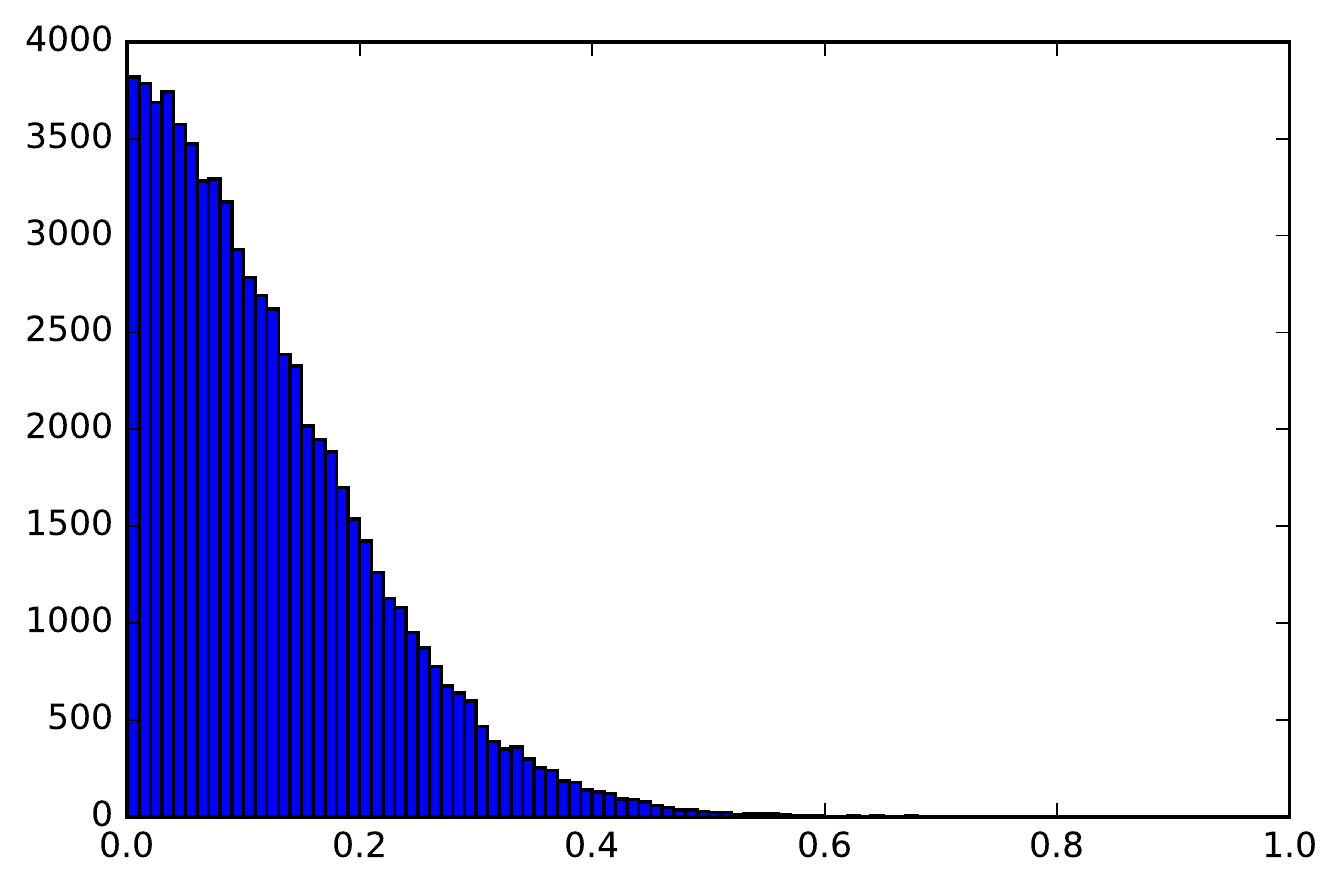} \hspace{6mm}  &
\includegraphics[width=0.28\textwidth]{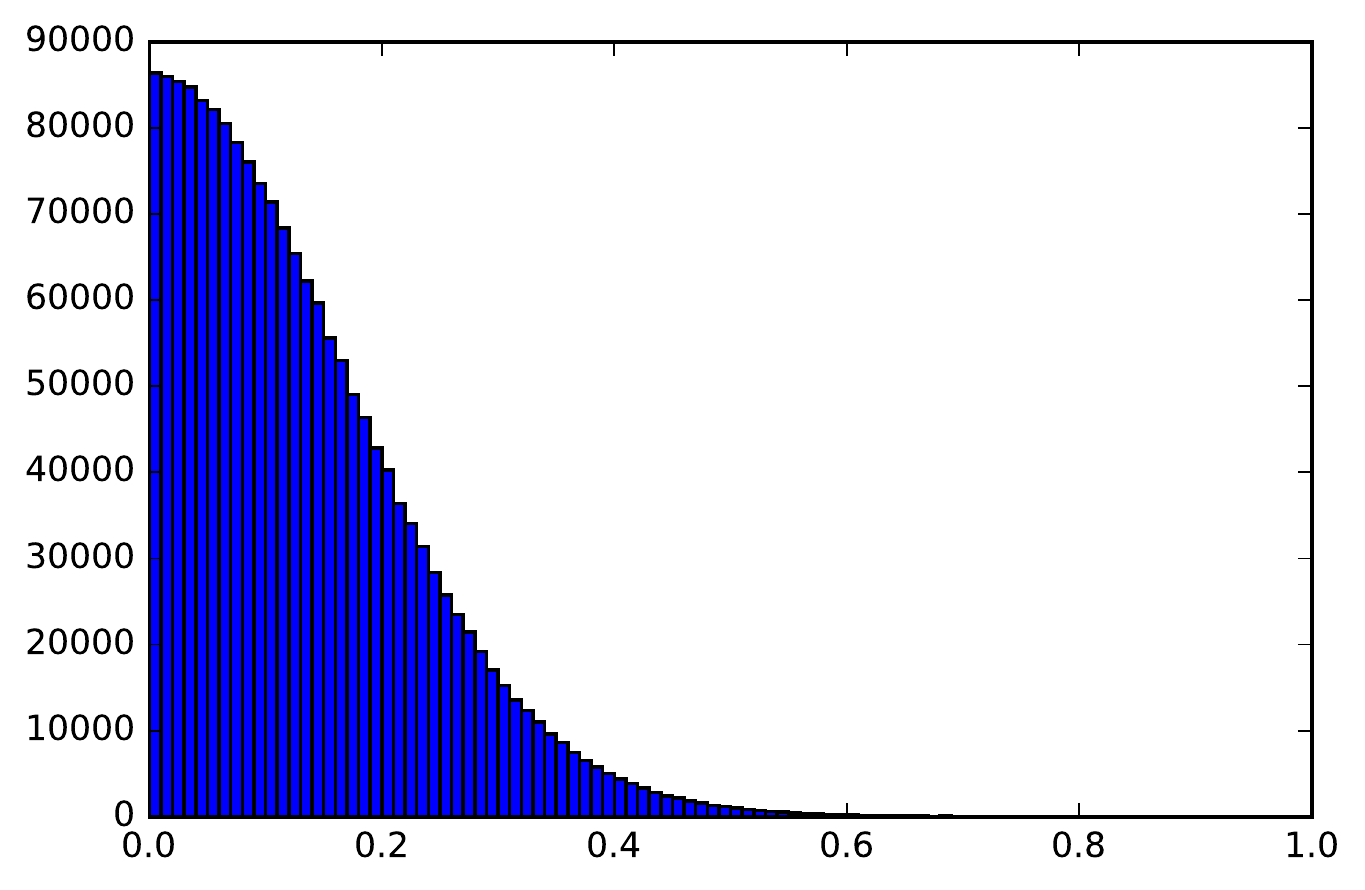} \hspace{6mm}  & \includegraphics[width=0.28\textwidth]{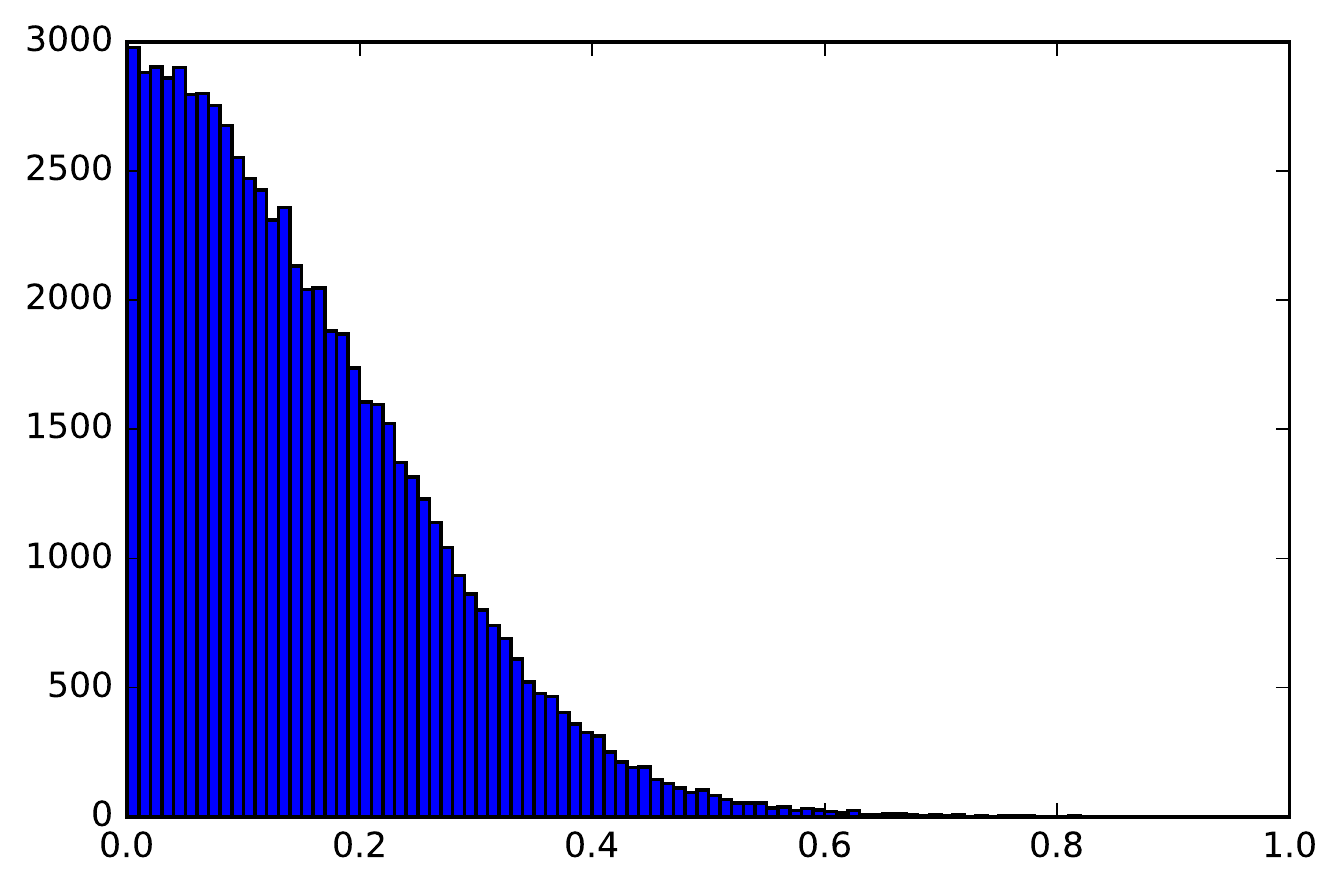} \\
(d) & (e) & (f) \\
\includegraphics[width=0.28\textwidth]{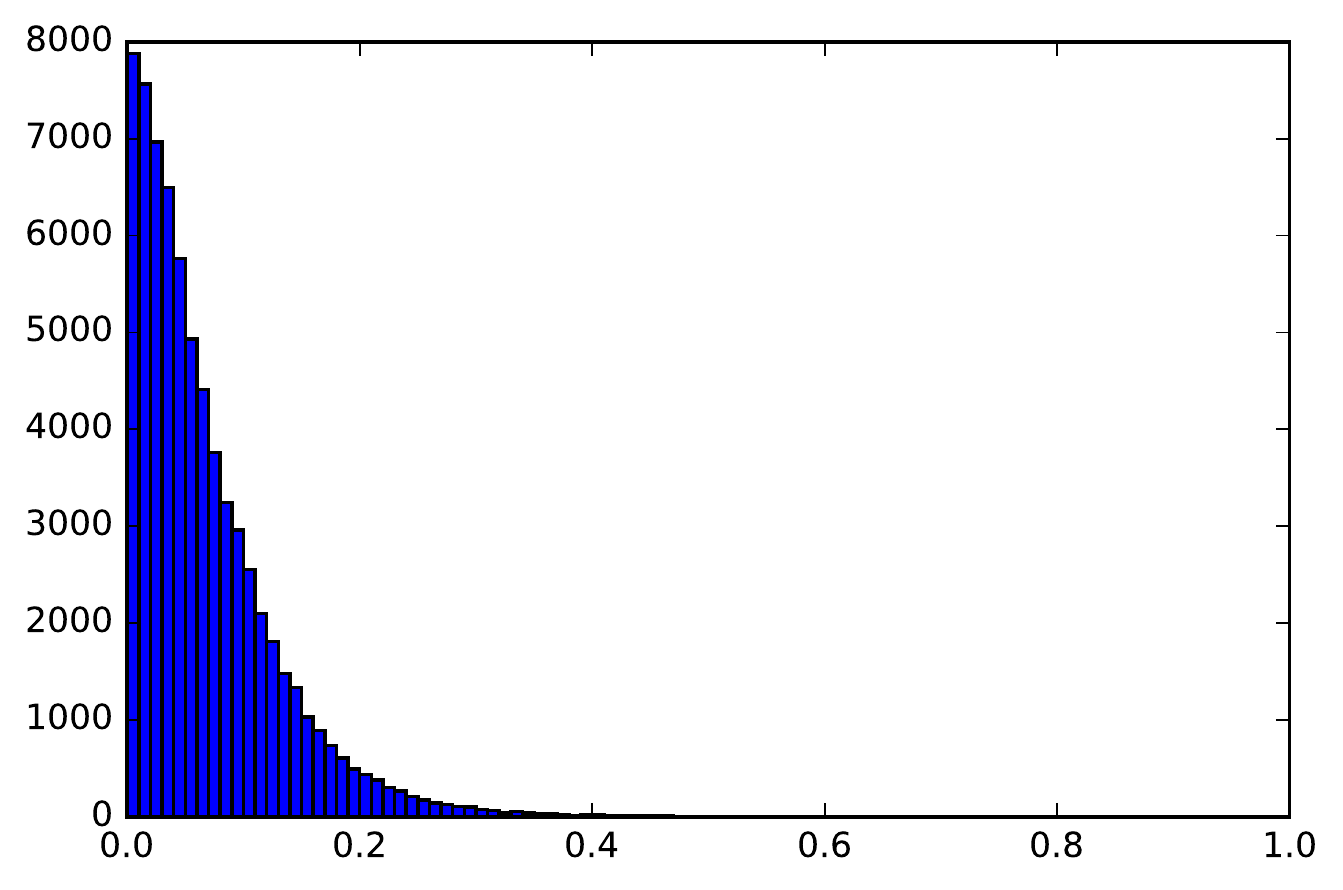} \hspace{6mm}  &
\includegraphics[width=0.28\textwidth]{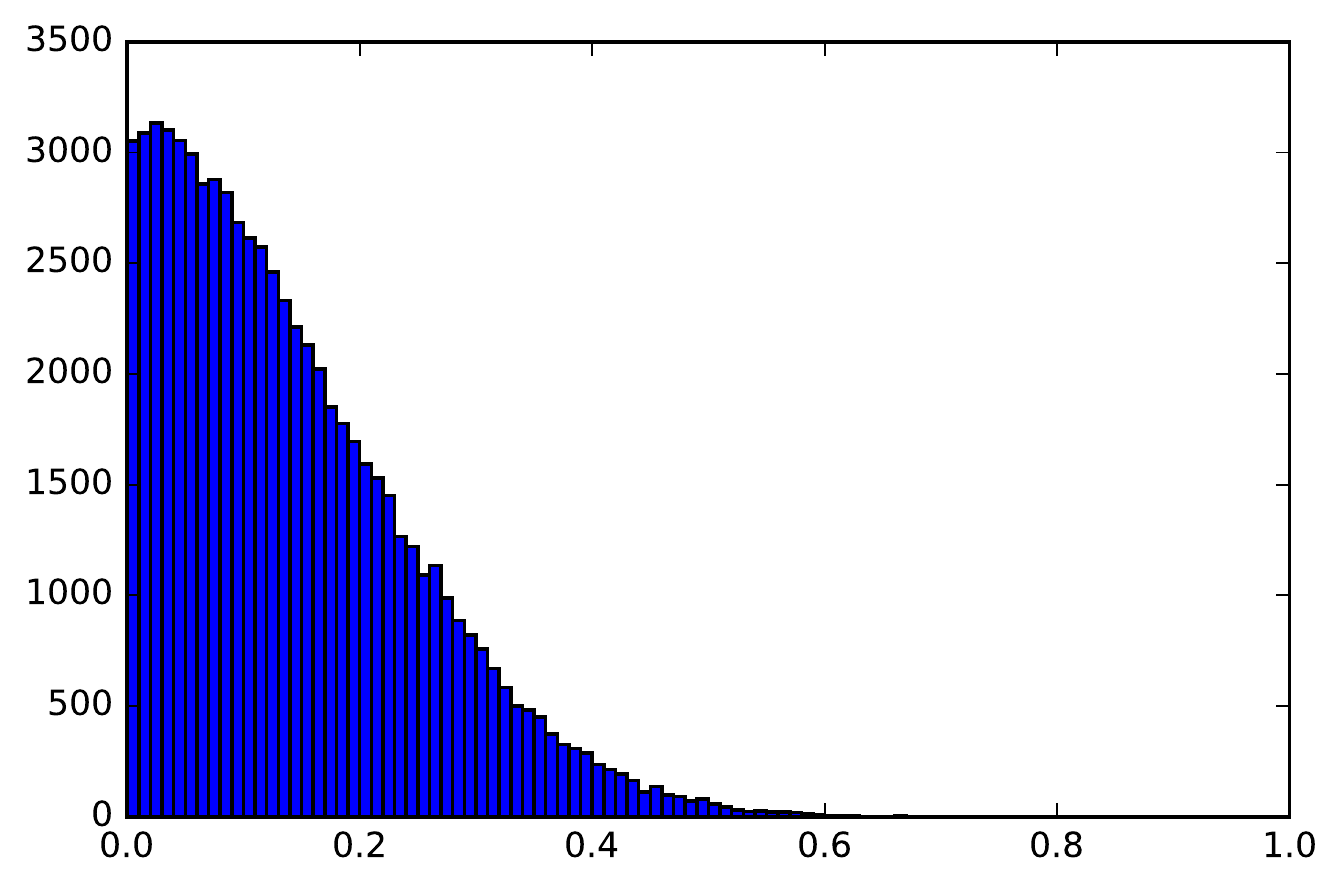} \hspace{6mm}  & \includegraphics[width=0.28\textwidth]{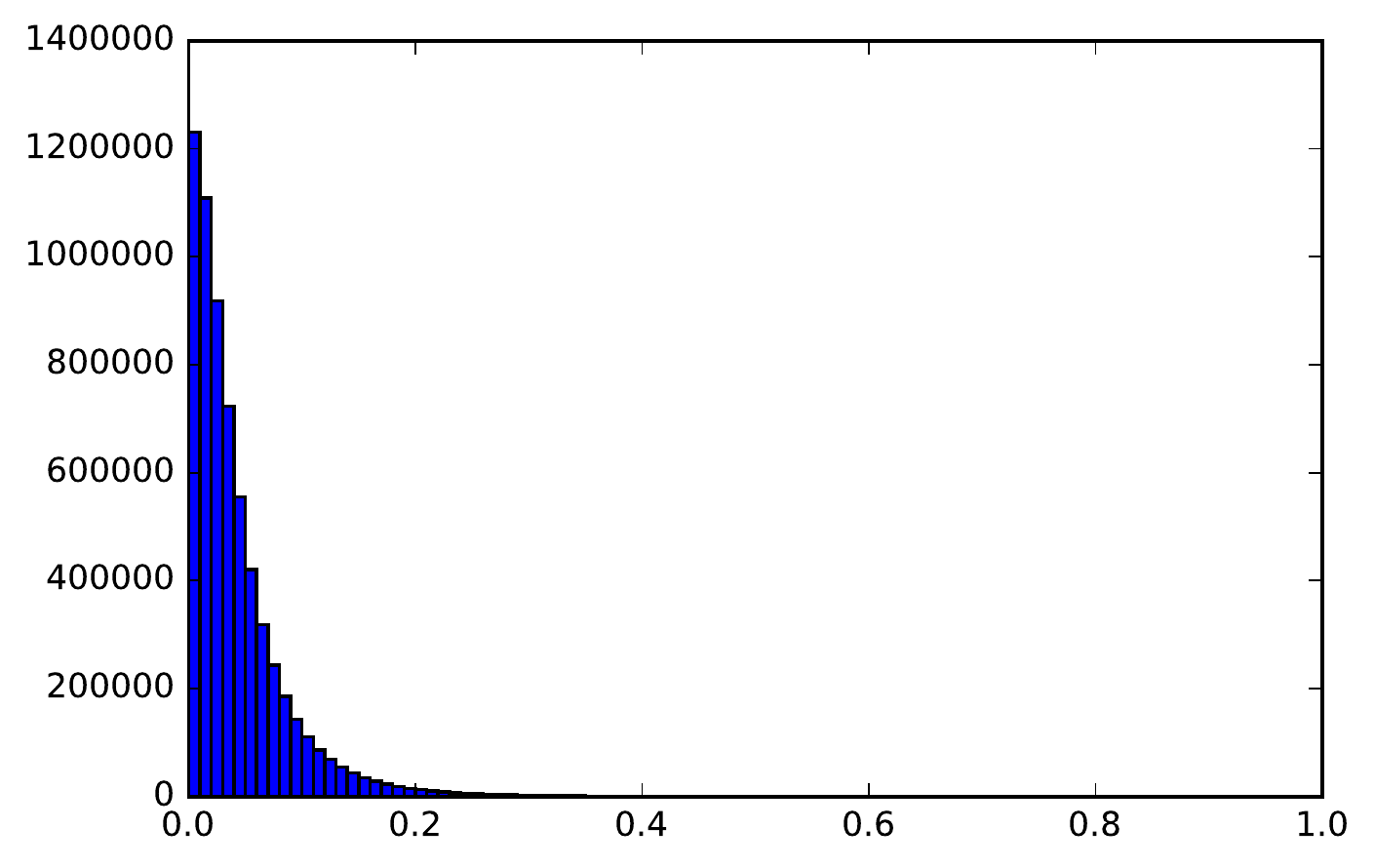} \\
(g) & (h) & (i) \\
\end{tabular}
\caption{\label{fig:distortion} Distribution of distortion for \methodname
$\left\{\left|\frac{\|\Pi x - \Pi y\|}{\|x - y\|} -1 \right| \right\}$ for all pairs of points $x,y \in T$. Here, the target dimension is set to 20. Datasets (a){\sc Computers} , (b) {\sc Earthquakes}, (c) {\sc FordB},
(d) {\sc LargeKitchenAppliances}  (e) {\sc Phoneme} (f)  {\sc RefrigerationDevices}, (g)  {\sc ScreenType}, (h) {\sc SmallKitchenAppliances},   (i)  {\sc UWaveGesture}. \methodname improves significantly PCA's distortions with negligible computational cost.} 
\vspace{-4mm}
\end{figure*}

While PCA provides no other guarantees concerning the distribution of pairwise distances for  an arbitrary cloud of points, empirically we find across a wide variety of datasets, including images, time series, gene data that the bulk of pairwise distances are preserved fairly well.    
Figure~\ref{fig:pcadistortion} shows the distortion that PCA causes for each pair of points  $\left\{\left|\frac{\|\Pi x - \Pi y\|}{\|x - y\|} -1 \right|  \right\}$ for all pairs of points $x,y \in T$ for nine datasets, see Table~\ref{tab:datasets}. Here, the target dimension of PCA is set to 20. We note that despite the relatively large  number of principal components several pairwise distances become significantly distorted.  We also note that PCA performs better on the dataset UWaveGesture, see Figure~\ref{fig:pcadistortion}(i), while for the rest of the dataset it does not perform as well. We investigate why this is the case, and we deduce a rule of thumb when PCA has a ``hard time'' to preserve pairwise distances.


For every  dataset in the collection, we compute its stable rank, defined as $\frac{(\sum \sigma_i)^2}{\sum \sigma_i^2}$. 
Stable rank is frequently used in order to evaluate  the intrinsic dimensionality of the data.  Figure~\ref{fig:stable} shows the stable ranks of the 9 datasets. It turns out that the dataset UWaveGesture has the smallest stable rank compared to the other datasets.  
As a rule of thumb we expect that when the stable rank is small few principal components will capture well the dataset, resulting in a better quality near-isometry.  On the contrary, when the stable rank is large, then we expect that a low-rank PCA approximation will distort significantly many pairwise distances. This is because the stable rank captures well how well the singular values are concentrated around their mean, and when it is large it implies that there is a lot of mass in the non-top principal components.  
 
%

\section{Experimental results}
\label{sec:exp}
\subsection{Experimental setup} 
\label{subsec:setup}
\begin{table}[!ht]
\begin{center}
\begin{tabular}{|l|c|c|} \hline
Name  & $n$ & $d$    \\ \hline 
\textcolor{green}{$\blacksquare$} Computers  & 250  & 720 \\
\textcolor{green}{$\blacksquare$} Earthquakes  & 139  & 512 \\
\textcolor{green}{$\blacksquare$} FordB  & 810  & 500 \\
\textcolor{green}{$\blacksquare$} LargeKitchenAppliances & 375  & 720 \\
\textcolor{green}{$\blacksquare$} Phoneme  & 214  & 1024 \\
\textcolor{green}{$\blacksquare$} RefrigerationDevices  & 375  & 720 \\
\textcolor{green}{$\blacksquare$} ScreenType  & 375  & 720 \\
\textcolor{green}{$\blacksquare$} UWaveGesture & 896  & 945 \\
\textcolor{blue}{$\bigtriangleup$} MNIST subset & 800  & 768 \\ 
\textcolor{cyan}{$\odot$}  Z\"{u}rich Database & 1\,005 & 307\,200  \\
\textcolor{red}{$\diamond $}   {\sc Glass } (6 Classes)    & 214  & 10  \\
\textcolor{red}{$\diamond$}   {\sc Ionosphere} (2 Classes) & 351  & 34 \\
\textcolor{red}{$\diamond$}  {\sc Iris}  (3  Classes) &150 & 4 \\
\textcolor{red}{$\diamond$}  {\sc Pima} diabetes (2 Classes) &768  & 8 \\
\textcolor{red}{$\diamond$}  {\sc Vehicle} (4  Classes) &846 & 18  \\
\textcolor{red}{$\diamond$}  {\sc Wine} (3 Classes) &178  & 13    \\
\hline
\end{tabular}
\end{center}
\caption{\label{tab:datasets} Datasets used in our experiments. }
\end{table}

\spara{Datasets.}  Table~\ref{tab:datasets} contains the description of the datasets we used. From the MNIST collection of hand-written letters~\cite{mnist}, a dataset with 60\,000 digit images, each of size $28\times 28$) , we selected 800 uniformly at random for our experiments. Unfortunately, NuMax takes an excessive amount of time ($>$ 9 hours) without completing.  
We use 9 datasets from the UCR time series collection~\cite{UCRArchive}. Again, for this collection of datasets NuMax does not complete in a reasonable amount of time.   
For our approximate nearest neighbor application we use  the ZuBuD image database \cite{shao03-zubud}. This  dataset contains 1\,005 images of 201 buildings in the city of Z\"{u}rich. There exist 5 images
of each building taken from different viewpoints, each of size 640 $\times$ 480 pixels. Finally we use 6  datasets from the UCI machine learning dataset collection \cite{UCI} for our classification experiments.

\spara{Implementation.} In our experiments, we use the following 
implementation for choosing the number of components for \methodname. For a given target dimension $r$, we use $\lfloor \frac{r}{2} \rfloor$ principal components for our matrix $P$ and $\lceil \frac{r}{2} \rceil$ JL components for our random projection matrix $S$. 
The resulting dimensionality reduction procedure   maps each vector $w$ in the cloud $T$ to the vector $(Pw, S(w - P^TPw))$. 
We have also experimented with using randomized PCA 
in place of the exact PCA using the approach proposed in \cite{Martinsson201147}. In our code we use the implementations provided by the {\it Sklearn} package. Our code  is available at \url{https://github.com/tsourolampis/Adagio}. 
Experiments were conducted on laptop with processor Intel(R) Core(TM) i5-3317U CPU @ 1.70GHz, and 8GB of RAM.

\spara{Synopsis of our findings.} Before we delve into the experimental findings we present a  synopsis of our findings in Table~\ref{tab:synopsis}. As we can see, our method \methodname combines the best of all worlds: data-awareness, near-isometry, and scalability.

\begin{table}[t]
\centering \small
\caption{\label{tab:synopsis} Synopsis of desired properties 
for random projections (JL), PCA, NuMax, and our method
\methodname.}  
\begin{tabular}{r|ccc|}
\multicolumn{1}{c}{} &  \multicolumn{3}{c}{Desired Properties} \\
\cline{2-4}
&
\multicolumn{1}{c}{Data-aware} &
\multicolumn{1}{c}{Runs fast} &
\multicolumn{1}{c|}{Near Isometry}  \\ \cline{2-4}
\textsf{JL}   & \xmark &   \cmark &   \cmark \\ 
\textsf{PCA}    & \cmark &   \cmark &   \xmark \\ 
\textsf{NuMax}    & \cmark &   \xmark &   \cmark \\ 
\textsf{\methodname}    & \cmark &   \cmark &   \cmark \\  
\cline{2-4}
\end{tabular}
\end{table}

\subsection{Data-aware Near-Isometry} 

Recall that Figure~\ref{fig:pcadistortion} shows the distribution of pairwise distortions for PCA when the target dimension is equally to 20, or equivalently when the cloud of points is projected on a 20 dimensional linear subspace.  The computational cost for running PCA is at most few seconds for all datasets. We apply \methodname on the same collection of datasets using the same target dimension. As we described in Section~\ref{subsec:setup}, \methodname uses 10 principal components and 10 JL components. Figure~\ref{fig:distortion} shows the distribution of distortion for \methodname
$\left\{\left|\frac{\|\Pi x - \Pi y\|}{\|x - y\|} -1 \right| \right\}$ for all pairs of points $x,y \in T$. The run times of \methodname are essentially identical with PCA. Specifically,  the computational overhead of \methodname is at most one second (UWaveGesture). 
As we will see later, there  exist also instances for which \methodname is faster  than PCA.  At the same time \methodname --as even eyeballing shows -- achieves a higher quality isometry than PCA.  NuMax does not produce any output in a reasonable amount of time. 

The contrast between Figures~\ref{fig:pcadistortion} and ~\ref{fig:distortion}   illustrates the importance of \methodname. Our method is able to preserve pairwise distances significantly better than PCA, is data-aware as it leverages the geometry through the top 10 principal components, and runs extremely faster than NuMax. The following experiments provide a detailed analysis of our method and its competitors.

\subsection{Dimension-distortion trade-off}

\begin{figure}
	\begin{center}
		\includegraphics[width=0.5\textwidth]{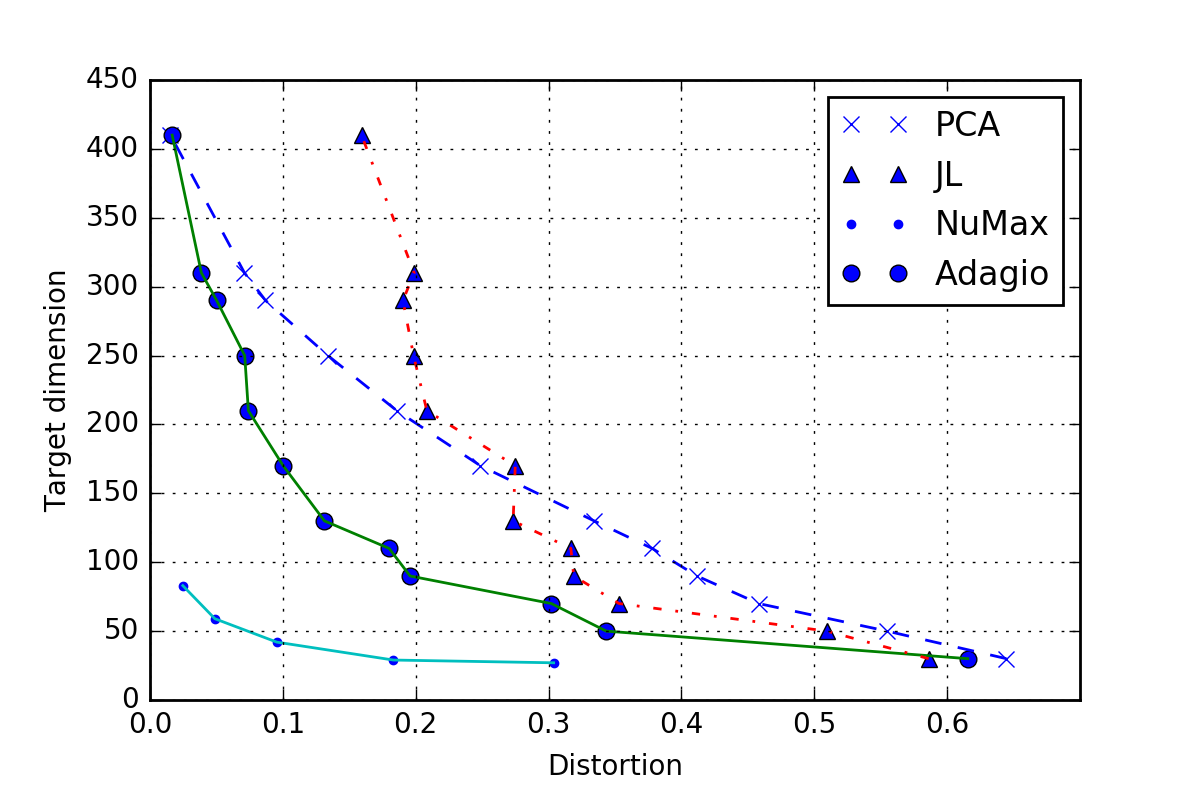}
	\end{center}
	\caption{
		\label{fig:mnist}
	Trade-off between target dimension and distortion for subset of $800$ pictures from MNIST dataset. }
\end{figure}

For a given distortion parameter $\delta$, what is the dimension of the subspace required by each method in order to achieve distortion at most $\delta$? This is one important question that we study empirically on the MNIST dataset. Specifically,  Figure \ref{fig:mnist} shows the experimental trade-off between the distortion $\delta$ of the embedding and the target dimension  $r$ for various dimensionality reduction methods. \methodname  yields a trade-off that consistently outperforms both Johnson-Lindenstrauss and PCA in the distortion regime up to $0.4$. In particular, in order to preserve all pairwise distances up to $15\%$ error, one can reduce dimension to $\sim 130$ with \methodname, whereas with PCA one needs dimension at least $240$.

NuMax algorithm finds an even better dimensionality reduction  than Adagio: for distortion $15\%$, one can use target dimension $\sim 50$. Unfortunately, this comes at a significant computational cost. The cost of computing such a reduction with NuMax is often prohibitively large,  usually by several orders of magnitude larger than Adagio, as we will see in the next Section where we study computational efficiency.

\begin{figure*}
	\begin{center}
		\begin{tabular}{cccc}
		\includegraphics[width=0.23\textwidth]{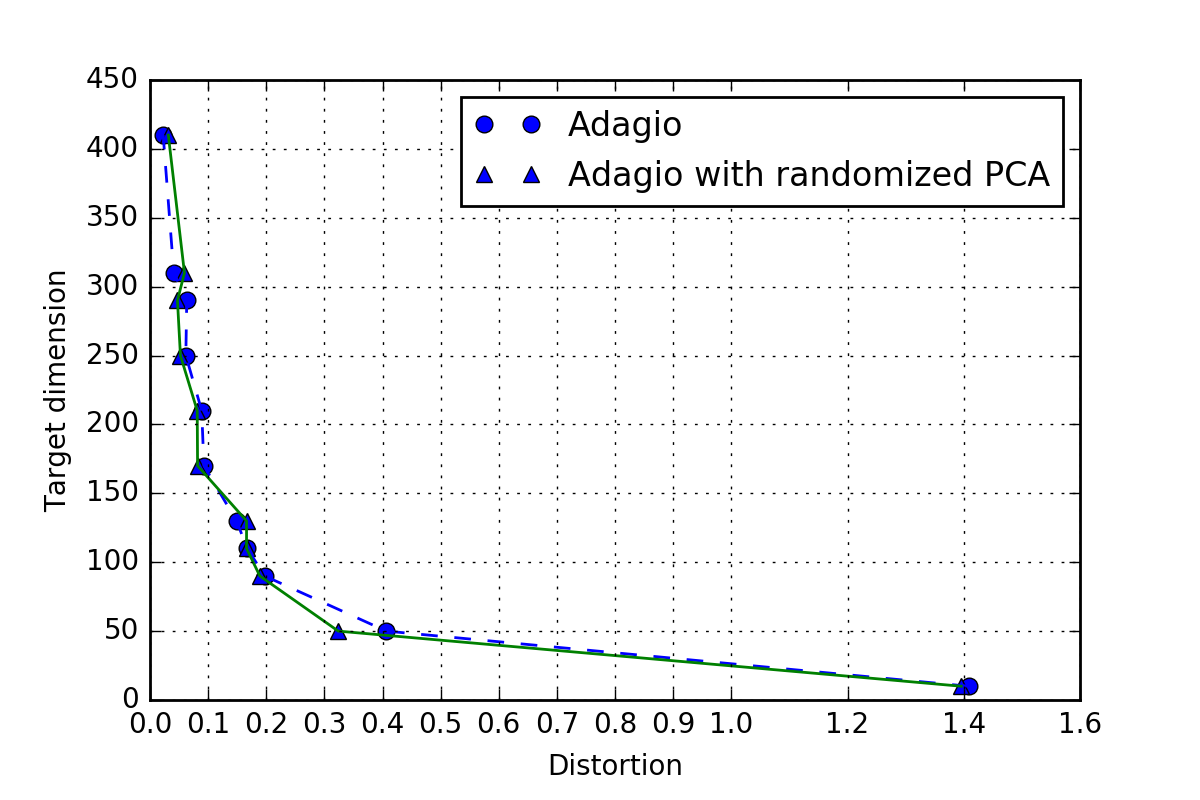} &
		\includegraphics[width=0.23\textwidth]{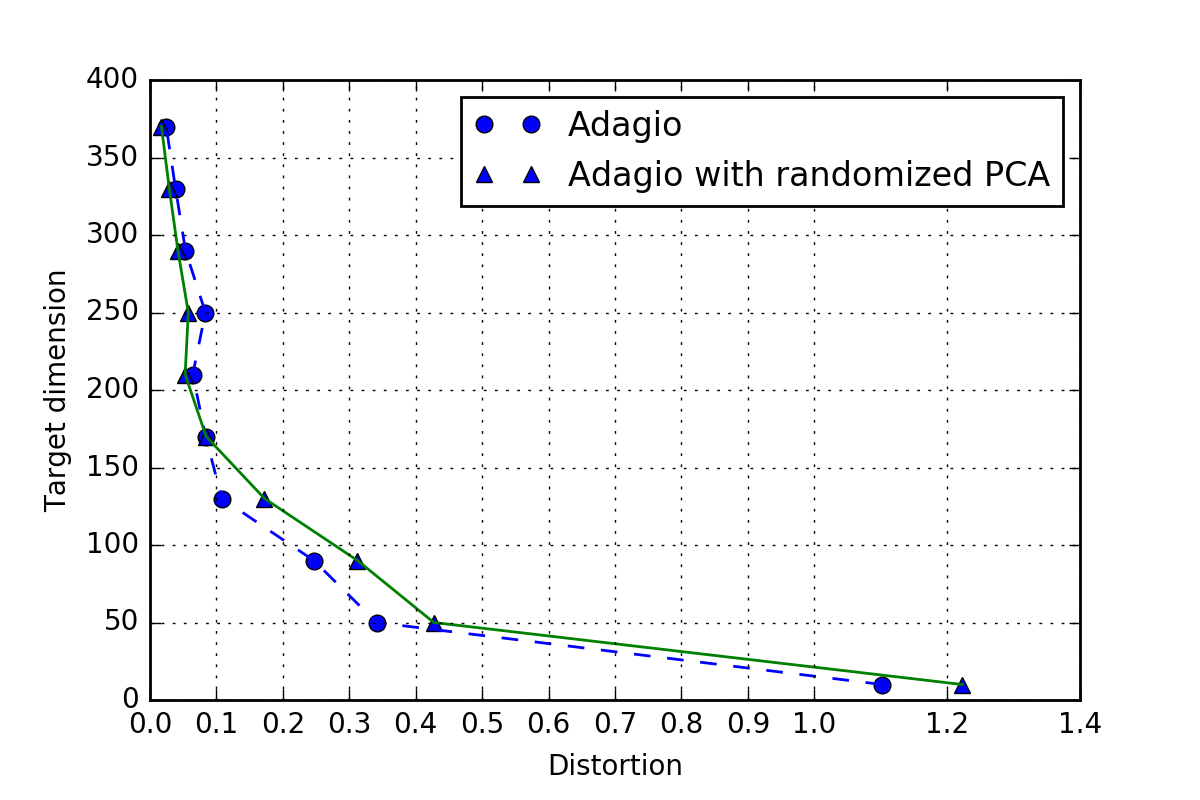}  & 
		\includegraphics[width=0.23\textwidth]{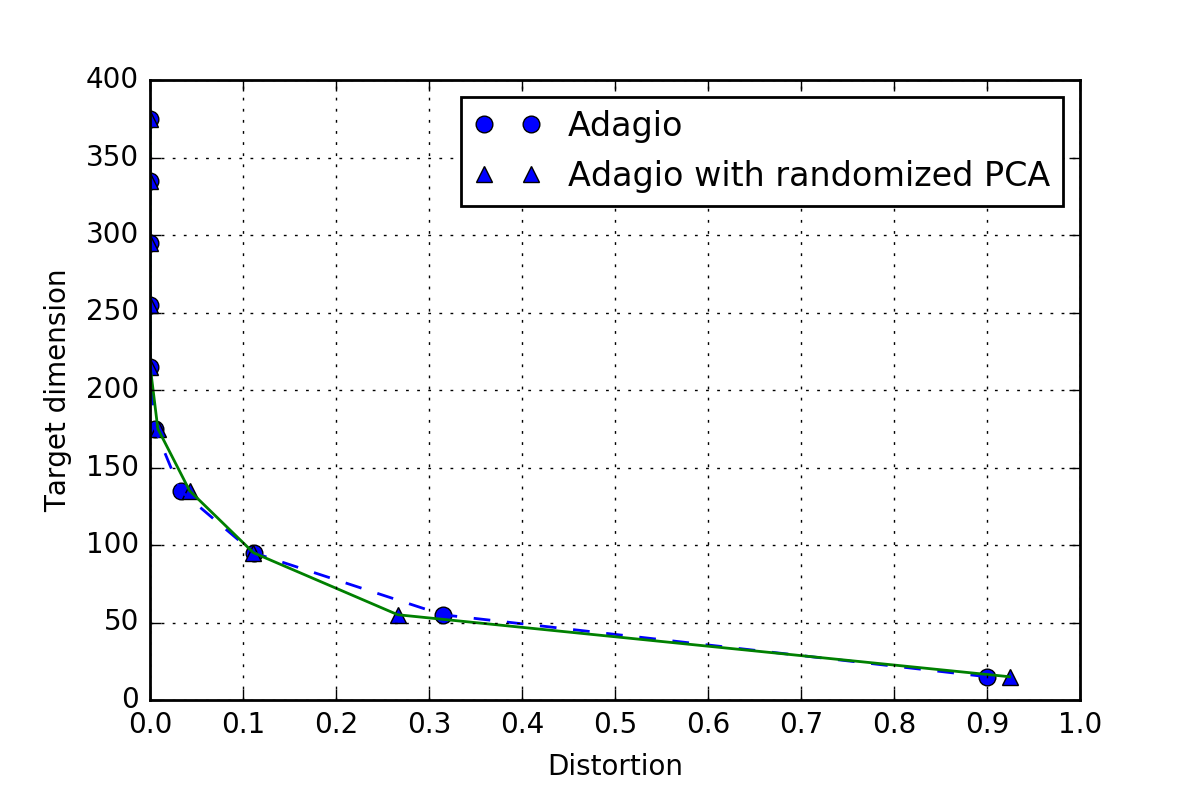} &
		\includegraphics[width=0.23\textwidth]{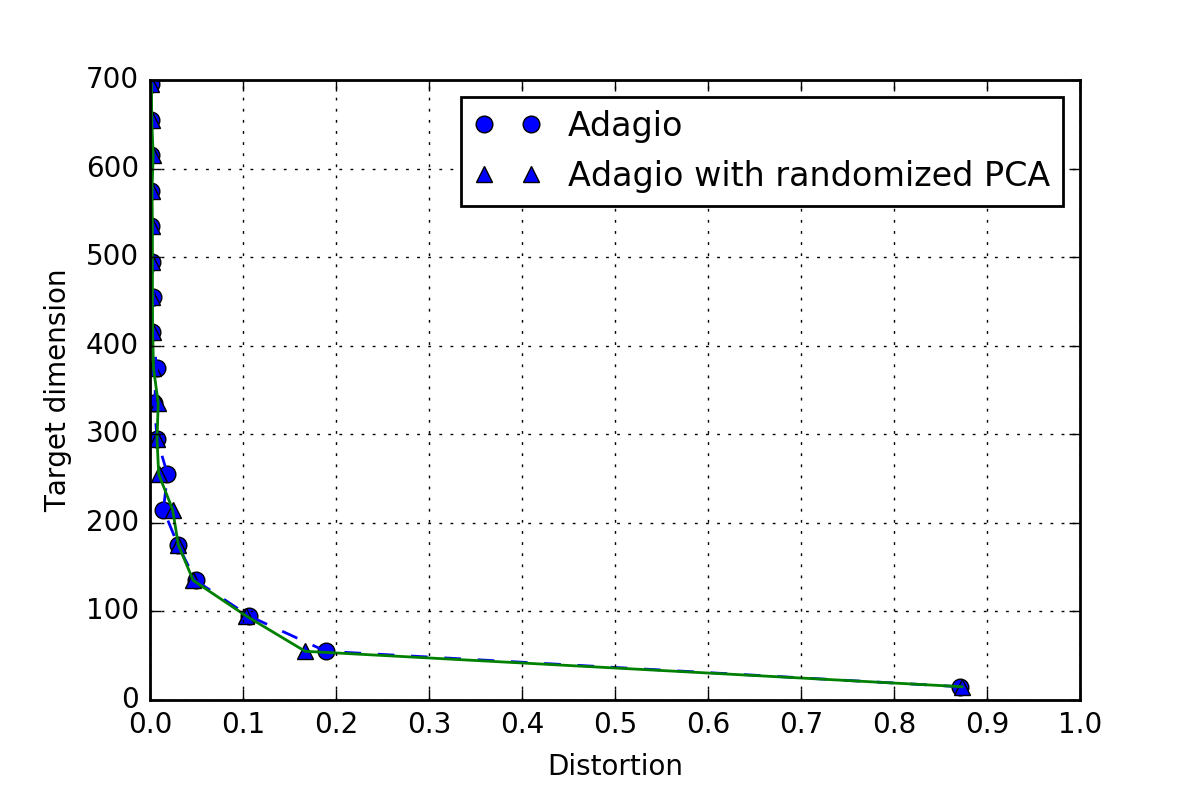} \\
		(a) & (b) &		(c) & (d)\\
				\includegraphics[width=0.23\textwidth]{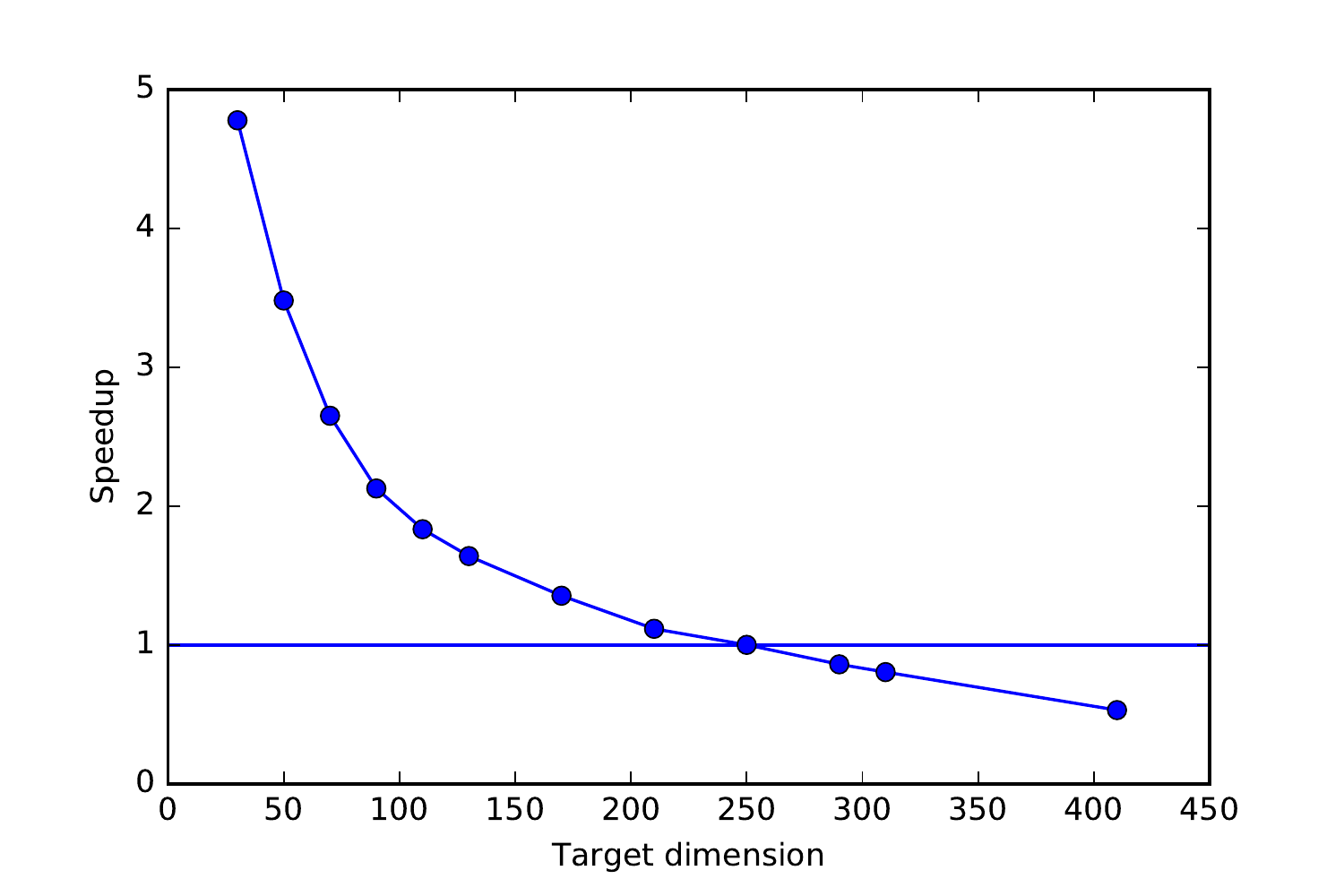} &
		\includegraphics[width=0.23\textwidth]{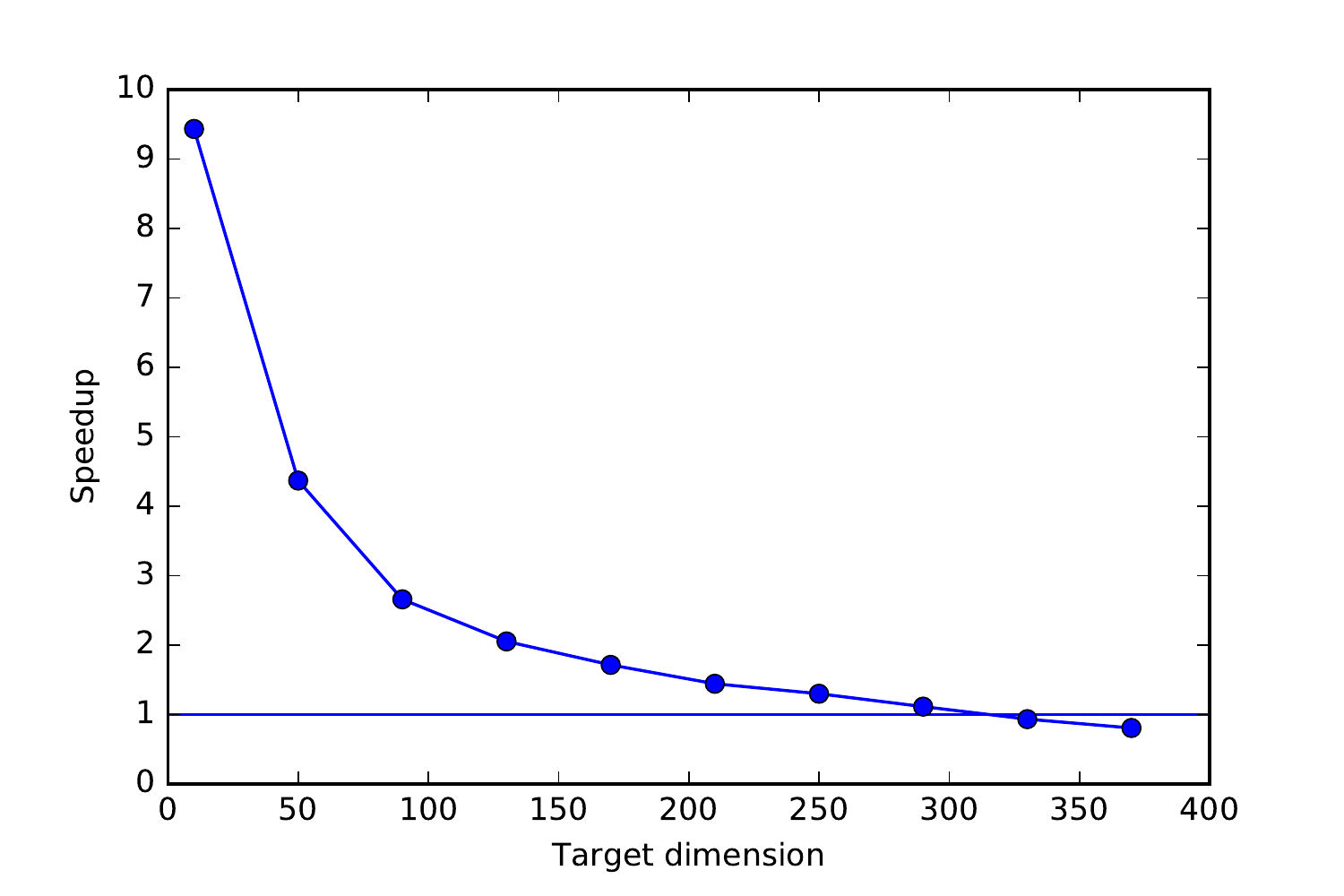}  & 
		\includegraphics[width=0.23\textwidth]{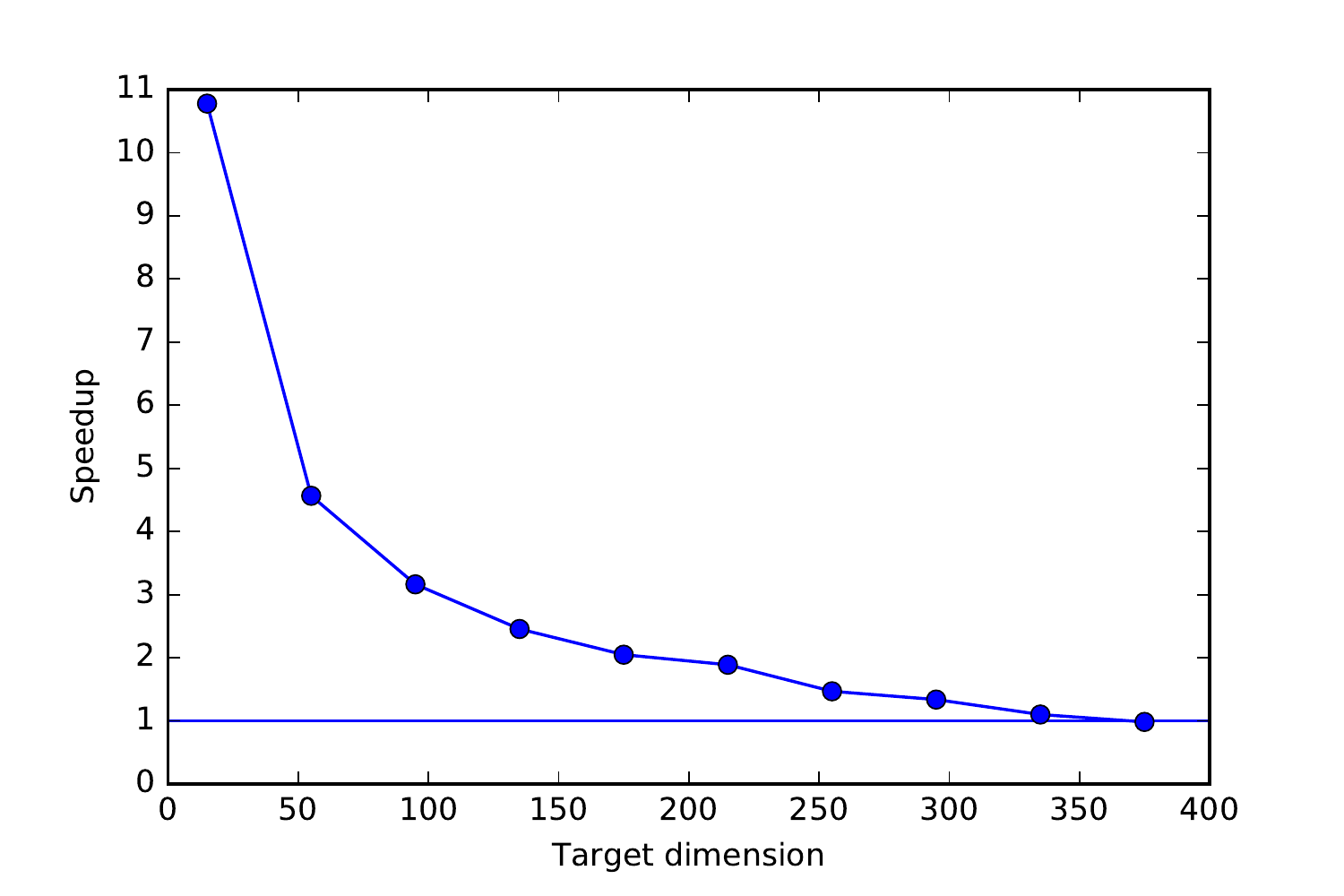} &
		\includegraphics[width=0.23\textwidth]{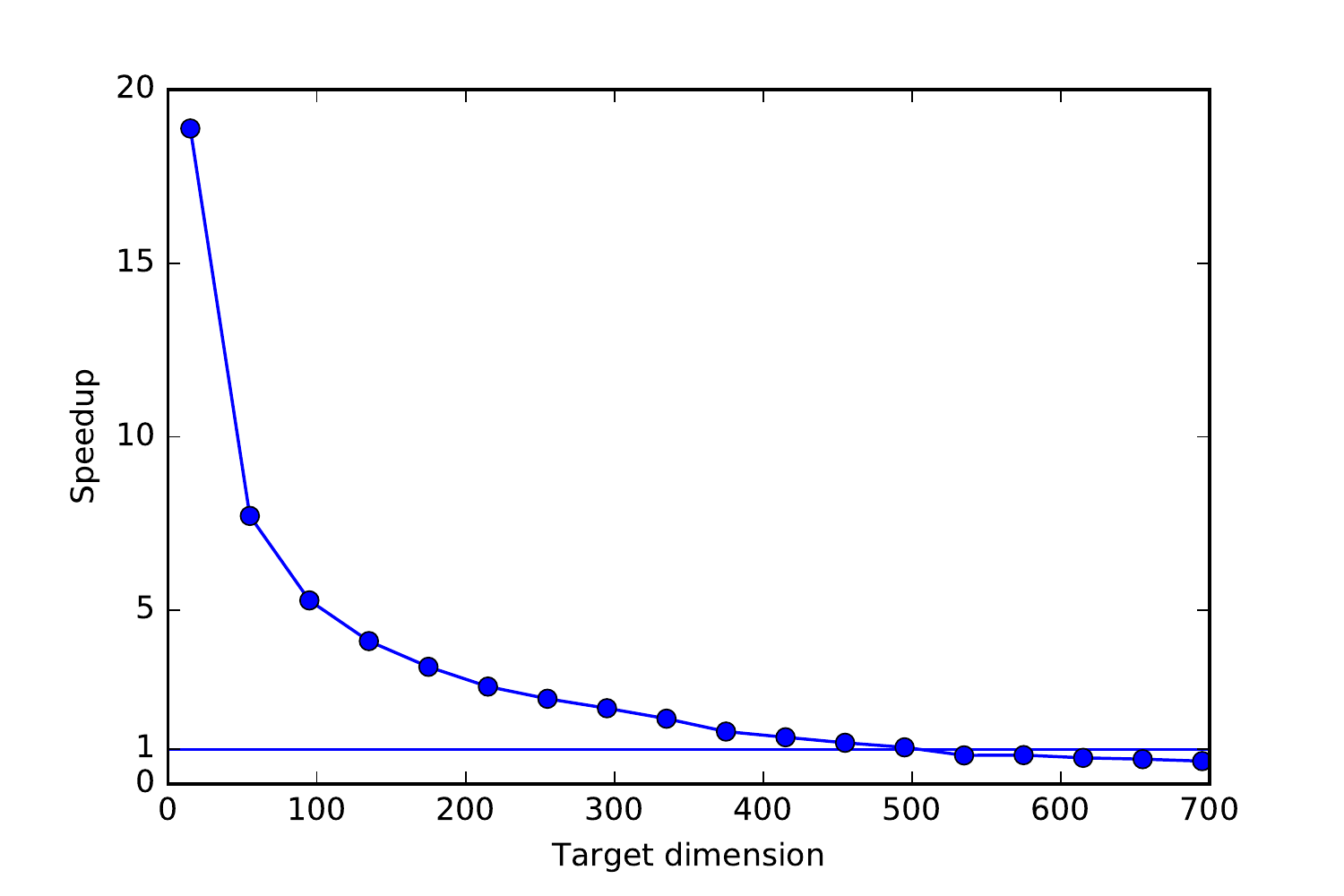} \\
		(e) & (f) & (g) & (h)\\
	\end{tabular}
	\end{center}
	\caption{
		\label{fig:rand-tradeoff}
	(I) Trade-off between target dimension and distortion on various datasets --- comparison between Adagio running PCA and randomized PCA as a subroutine --- usage of randomized PCA does not have a significant effect on quality of solution.  Datasets (a) MNIST-subset, (b) ScreenType, (c) FordB, (d) UWaveGesture. 
 (II) Corresponding speedups  for 	(e) MNIST-subset, (f) ScreenType, (g) FordB, (h) UWaveGesture. 
	 }
\end{figure*}

\subsection{Computational efficiency}
\label{subsec:efficiency}

Table~\ref{tab:efficiency} presents our findings on a randomly chosen sample of 800 points from the MNIST dataset for random projection, PCA,   NuMax, \methodname using  exact PCA as its subroutine  (Adagio), and \methodname using randomized PCA \cite{randomizedPCA} (Adagio-Randomized-PCA). For each maximum distortion level $\delta \in \{0.05, 0.1, 0.2\}$   we compute how many dimensions are needed  per  method to achieve this near-isometry quality. Note that for random projections there is no information for $\delta$ equal to 0.05 and 0.1. This is because for these low levels of distortion, the target dimension is equal to the ambient dimension 784, i.e., there is no dimensionality reduction.   

As in the previous section, we observe that NuMax is able to detect lower dimensional subspaces compared to the rest of the methods for a given level of distortion. Nonetheless, the computational cost is excessive. Even for the subset of 800 points, it takes more than 18 minutes (1\,105 seconds) to achieve distortion equal to 0.05.  When NuMax runs on the full MNIST dataset with 60\,000 digits, it takes more than 9 hours without completion. Also, note that as the distortion level decreases from 0.2 to 0.05 the run time increases rapidly. 
Our method \methodname provides a nice trade-off between  computational cost and the quality of the near-isometry. It   outperforms random projections and PCA by simply being their combination. 

\begin{table}
	\begin{center}
	\begin{tabular}{l|ccc}
		Method & Distortion & Dimension & Time \\ \hline
		
		\multirow{3}{*}{Random Projection (JL Lemma)} & 0.05 & -- & -- \\
		& 0.1 & -- & -- \\
		& 0.2 & 260 & 0.45s \\ \hline
		
		\multirow{3}{*}{PCA} & 0.05 & 337 & 5.18s \\
		& 0.1 & 280 & 5.1s \\
		& 0.2 & 205 & 4.84s \\ \hline
		
		\multirow{3}{*}{NuMax} & 0.05 & 83 & 1105s \\
		& 0.1 & 59 & 373s \\
		& 0.2 & 42 & 220s \\ \hline

		\multirow{3}{*}{Adagio-Exact-PCA} & 0.05 & 298 & 5.6s \\
		& 0.1 & 187 & 5.5s \\
		& 0.2 & 95 & 4.8s \\ \hline

		\multirow{3}{*}{Adagio-Randomized-PCA} & 0.05 & 298 & 6.64s \\
		& 0.1 & 190 & 4.01s \\
		& 0.2 & 98 & 1.82s \\  
 
	\end{tabular}
	\end{center}

	\caption{\label{tab:efficiency} Run time for dimensionality reduction on 800 sample points from MNIST dataset. For details, see Section~\ref{subsec:efficiency}.}
\end{table}

 Additionally, \methodname using randomized PCA as its black-box rather than exact PCA performs comparably well with \methodname. The astute reader may note that  for distortion $\delta=0.05$ \methodname with exact PCA versus \methodname with randomized PCA is faster. This happens because the constants hidden in the big-O notation asymptotics and the number of principal components required. Specifically, randomized PCA requires time $O(mn \log{k} + (m+n)k^2 )$ to compute the top $k$ principal components for a matrix which is $m\times n$. Exact PCA requires time equal to $O(nmk)$. It turns out that despite the better asymptotics of randomized PCA, exact PCA wins in practice  for this specific experiment.  
In general, the use of randomized PCA speeds up \methodname while maintaining the output quality.  We explore randomized PCA 
as a subroutine for \methodname in the following.

\subsection{\methodname with Randomized PCA}

Figure~\ref{fig:rand-tradeoff}(a),(b),(c),(d) plot the target dimension $r$ versus the resulting distortion for \methodname when it uses an exact and a randomized PCA subroutine to find the top principal components for   the MNIST subset,   ScreenType,   FordB, and UWaveGesture
respectively. We observe that the quality of the output remains almost unchanged when we use randomized PCA. Figure~\ref{fig:rand-tradeoff}(e),(f),(g),(h) plot the resulting speedup respectively. We observe that when the number of principal components is relatively small, the resulting speedups are important. As the target dimension increases, the speedups decrease. The  run times without randomized
PCA are  5.5, 1.2, 2.0, and 8.3 for the four datasets respectively. 
Finally, the results shown in Figure~\ref{fig:rand-tradeoff} are representative of our findings across all datasets.

\subsection{Approximate Nearest Neighbor Queries}
\label{subsec:nearestneighbors}
\begin{figure*}
	\begin{center}
		\begin{tabular}{cc}
		\includegraphics[scale=0.5]{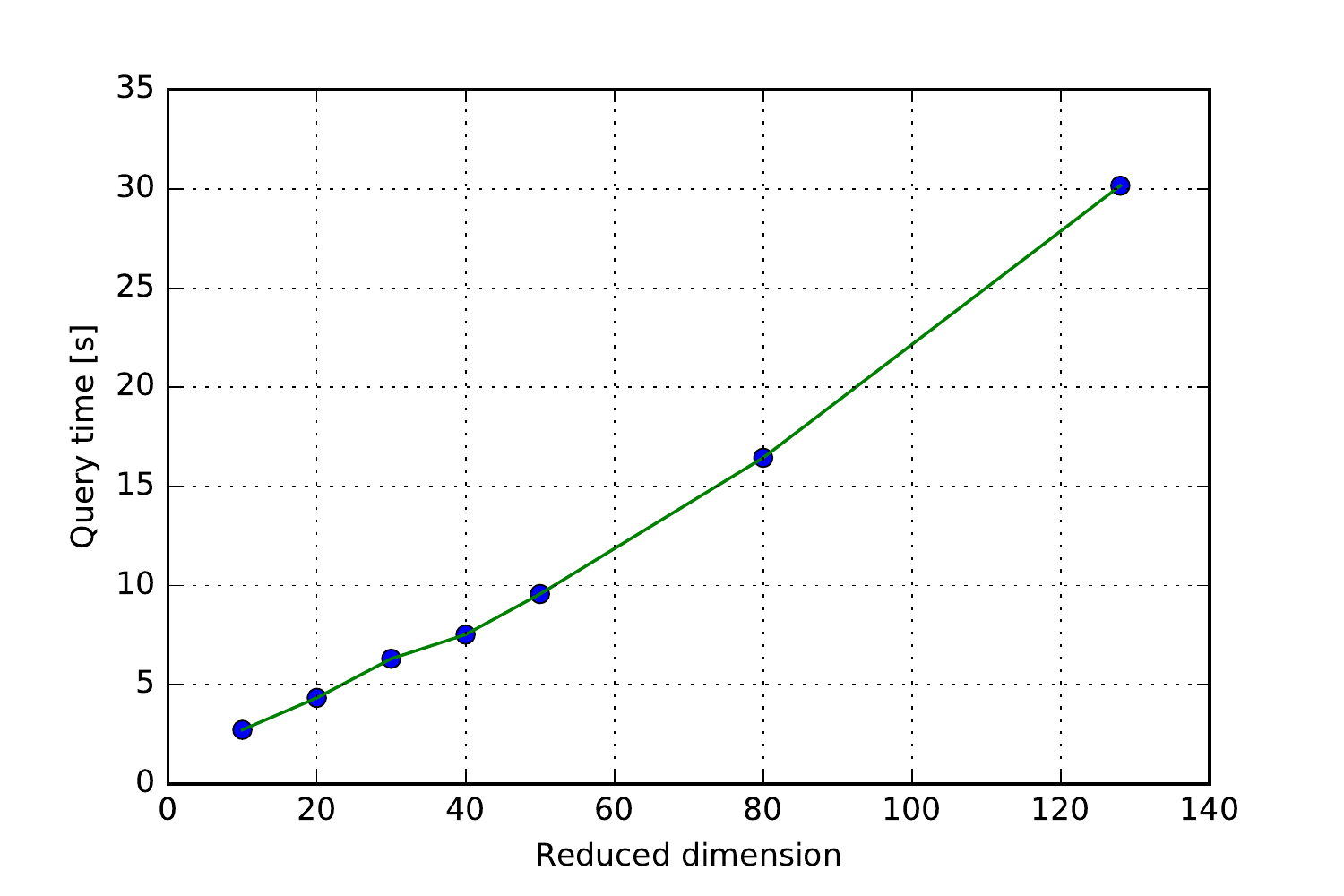} &
		\includegraphics[scale=0.5]{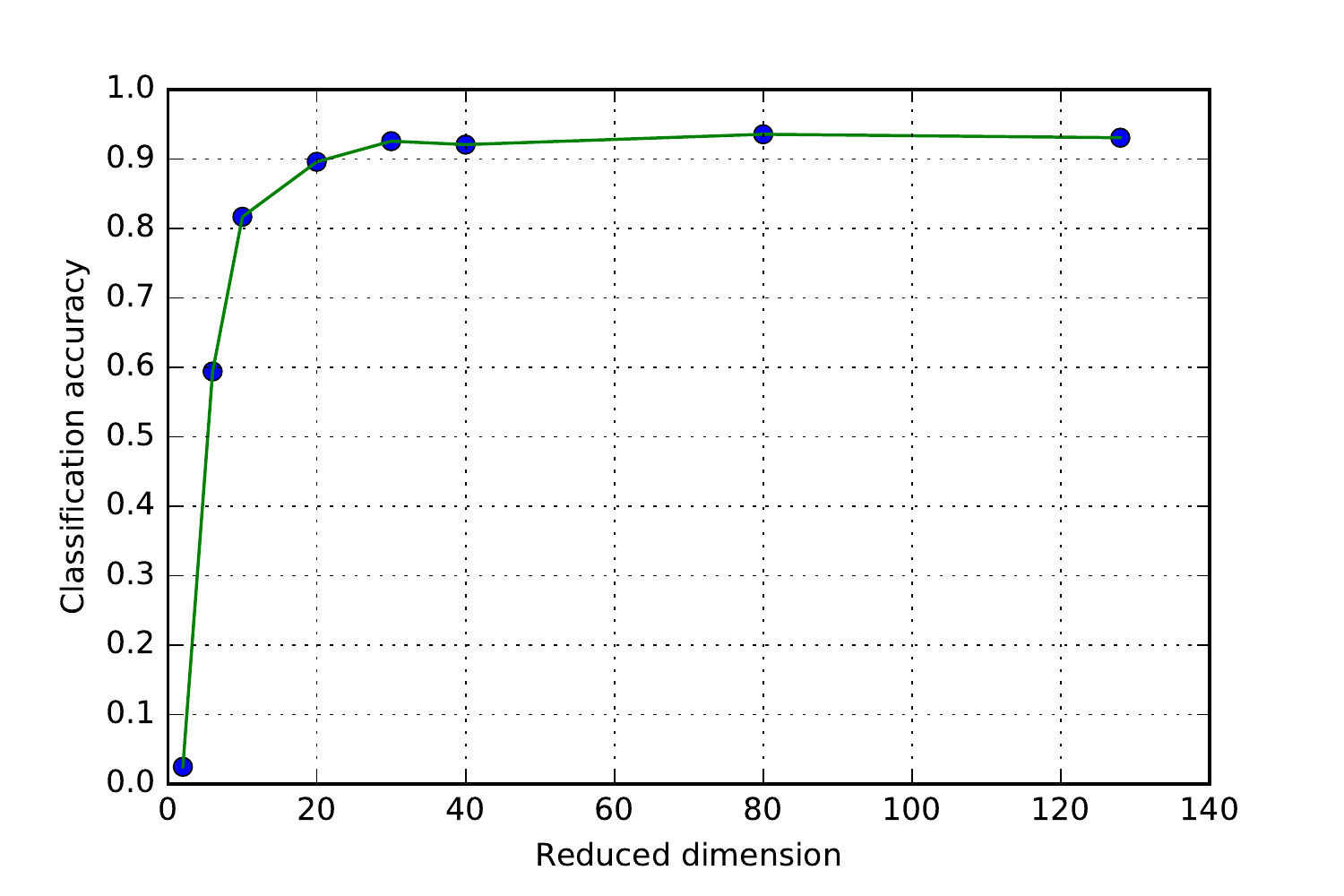} \\
		(a) & (b)
	\end{tabular}
	\caption{ 	\label{fig:query_time_reduction}
		Query time and accuracy for classification of building photos from ZuBuD database using \methodname for different target dimensions. The ambient dimension is 128. For details, please see Section~\ref{subsec:nearestneighbors}.
	}
	\end{center}
\end{figure*}

We explore the usage of \methodname as a preprocessing step for nearest neighbor queries. Such queries play a major role in numerous applications and have attracted significant research interest \cite{andoni2006near,arya1993approximate,ding2004k,ide2007computing,gionis1999similarity,indyk1998approximate,kleinberg1997two,ueno2006anytime}.

Our running application is image retrieval. We use the ZuBuD database \cite{shao03-zubud}, with pictures of 201 buildings in Z\"{u}rich; for each building 5 different photos are provided, with different viewpoints and light conditions. Thus, in total the image database consists of 1\,005 pictures.  We use the well-established SIFT features \cite{Lowe2004} that provide significant invariance properties for image retrieval applications \cite{boufounos2015representation,LiRB}. 

 We create a training and a query/test dataset as follows. For each one of 201 buildings, we pick the first picture --- all these pictures will form a query/test set. The four remaining pictures for each building are included in the training dataset. 

Using the SIFT algorithm we extract $500$ scale-invariant \emph{descriptor vectors} for every image in the training dataset.
Each  \emph{descriptor vector} is a point in $\field{R}^{128}$. 
At the end of processing the data, we have a \emph{database} $\mathcal{D}$ with $4 \text{~pictures~} \times 201 \text{~buildings~} \times 500 \text{~descriptor vectors~}$ points in $\mathbb{R}^{128}$.

In order to decide  which building a test photo depicts, we extract from this photo $100$ scale-invariant  descriptor vectors in $\mathbb{R}^{128}$ using again the SIFT algorithm. For each descriptor vector we perform a 10 nearest neighbor query to our 
\emph{database} $\mathcal{D}$. Thus, for any given query image, we create a multiset of descriptor vectors. We classify the picture as the building that appears most often in the resulting multiset. If a tie occurs, then we perform random assignment to one of the tied candidate buildings.  

We perform the same experiment for different target dimensions, and we evaluate the performance via  the accuracy defined as

$$ \text{accuracy} = \frac{ \text{\# correctly classified buildings}}{ \text{\#  of  buildings}}.$$

Using the aforementioned experimental setup, we find that without reducing the dimensionality of the descriptor vectors, we can achieve 
classification accuracy $93\%$. Spefically, without  any dimensionality reduction method $188$ out of $201$ objects are correctly classified using this simple majority rule.

We use \methodname to reduce the dimensionality of the descriptor vectors to different target dimensions.
Figure~\ref{fig:query_time_reduction}(a) shows the effect of reducing the dimension of the SIFT descriptor vectors on query times. We verify the fact that the smaller the dimension of the points, the faster we can answer a given image retrieval query.  On the other hand, Figure~\ref{fig:query_time_reduction}(b) shows that as the dimensionality decreases, so does the accuracy. Our findings indicate that the critical dimension value is around 20. For any target dimension 
greater than 20, we do not observe any significant loss in the classification accuracy.

Figure~\ref{fig:zubud-sample-images}(a) shows a sample query image, together with three closest candidates from the training dataset. It is worth noting that the building in Figure~\ref{fig:zubud-sample-images}(d) 
has significantly less votes compared to  buildings (b) and (c).
When we apply \methodname with target dimension greater than  20, 
Figure~\ref{fig:zubud-sample-images}(a) is correctly classified 
 to represent the same building as in Figure~\ref{fig:zubud-sample-images}(b). When the target dimension becomes less than 20, it is confused with Figure~\ref{fig:zubud-sample-images}(c).

\begin{figure*}[t]
	\begin{center}
		\begin{tabular}{cccc}
			\includegraphics[scale=0.15]{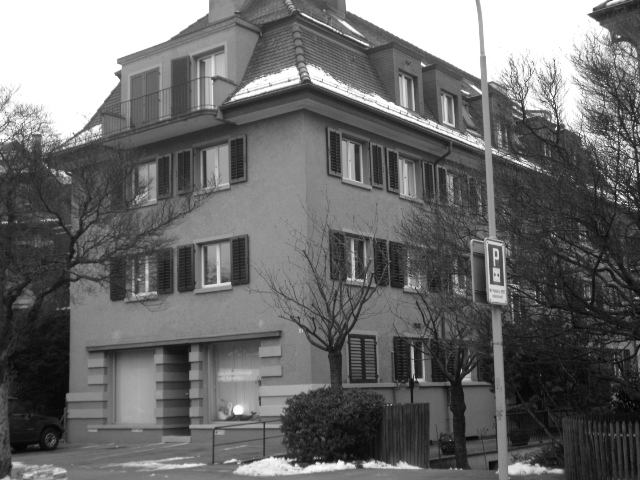} \hspace{6mm} &
			\includegraphics[scale=0.15]{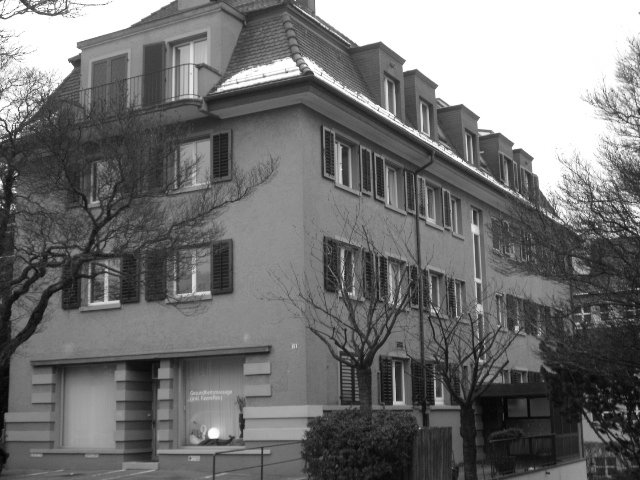} \hspace{6mm}  & 	
			\includegraphics[scale=0.15]{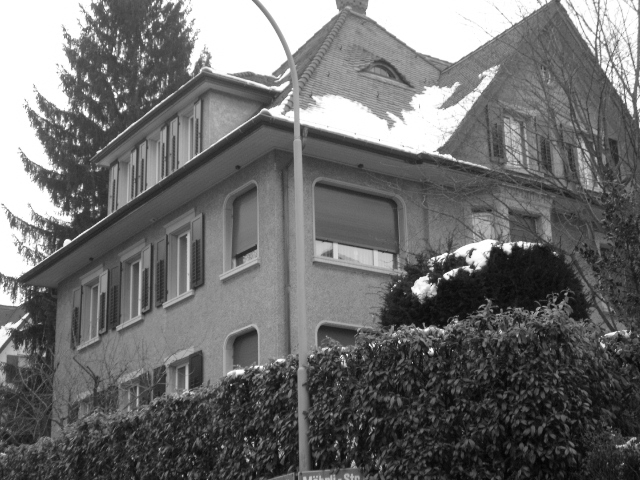} \hspace{6mm} &
			\includegraphics[scale=0.15]{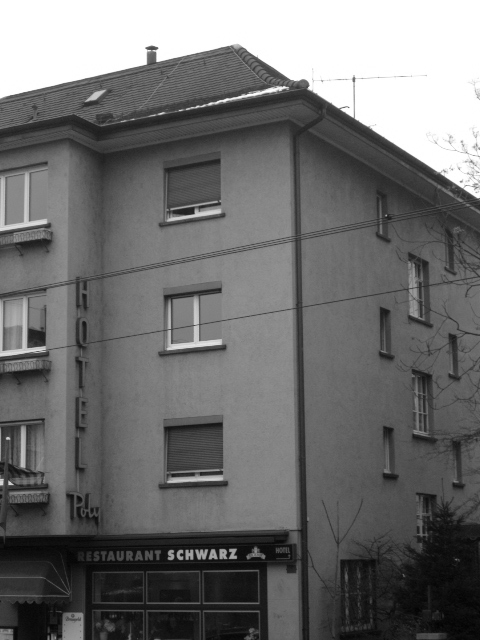} \hspace{6mm} \\
			(a) & (b) 		&			(c) & (d)\\
		\end{tabular}
		\caption{Sample images from ZuBuD database. (a)  Query image, (b), (c), (d)  nearby buildings in the training dataset. When the target dimension of the dimensionality reduction is $>20$, (a) is correctly classified as the same object as (b). After reducing dimension to $ r \leq 20$, it becomes incorrectly classified as (c) \label{fig:zubud-sample-images}.}
	\end{center}
\end{figure*}

\subsection{Classification} 
\label{subsec:classification}
We conclude our applications by studying the effect of dimensionality reduction on classification accuracy.  For  timing aspects of nearest-neighbor queries, our results are consistent with the detailed results of the approximate nearest neighbor application.  We perform binary classification  experiments as follows.  For a given training dataset of vectors indexed by the set $S \subseteq [n]$, let $c_i \in \{0,1\}$ be the label of $i \in S$.  We apply the semi-supervised learning algorithm due to Zhu, Ghahramani, and Lafferty  \cite{zhu2003semi} for learning the labels of vertices in the graphs. Specifically, we solve for the vector $x$ that minimizes 

\begin{equation} 
\label{lapl} 
x^T L x = \sum_{(u,v) \in E(G)} w_{uv} (x_u-x_v)^2,
\end{equation}

\noindent subject to $x_i=c_i$ for all $i  \in S$. This can be encoded as a symmetric diagonally dominant linear system which can be solved in theory more efficiently even than sorting \cite{cohen2014solving}. However in our experiments, we use usual matrix inversion. For each node $j \notin S$ we decide  $c_j=1$ if $x_j \geq \frac{1}{2}$, and otherwise we set $c_j=0$.  We perform 10-fold cross-validation and we report the mean classification error of the classifier, defined as  
$ \text{ {\sc err} } = \frac{ \text{fp+fn} }{\text{tp+tn+fp+fn}}$, where $\text{tp,tn,fp,fn}$ are the number of true positives, true negatives, false positives, and false negatives respectively. When we have more than two classes, we consider each class separately and we group the points from the rest of the classes as a single class. We then solve \eqref{lapl} for each class $c$ to obtain a vector $x^c$. We decide that an unlabeled point $i$ belongs to the class $c$ 
that maximizes $x^c_i$. We evaluate the classification output with  
the mean classification error, namely the fraction of misclassified points. 

Table~\ref{tab:classification} shows the results we obtain using \methodname and NuMax respectively. For each dataset we record the time $T_1$ required to reduce the dimensionality, the time $T_2$ required to construct the $k$-nearest-neighbor graph for $k=30$ for each dataset, and the resulting classification accuracy. We also run Zhu et al. classification algorithm on the original dataset to understand the effect of dimensionality reduction on classification accuracy.   The resulting accuracies are 80.89\%, 59.79\%, 56.67\%, 62.93\%, 74.15\%, and 63.59\% for Glass, Ionosphere, Iris, Pima Diabetes, Vehicle, and Wine datasets respectively. By comparing these accuracies with the accuracies obtained by using NuMax and \methodname as preprocessing steps, we observe that typically the accuracy does not change significantly, and in some cases it may even increase. Also, we observe again that $T_1$ is larger for NuMax compared to our method.

\begin{table}[t]
\centering \small
\caption{\label{tab:classification} NuMax and \methodname performance for the  classification task using $k$-NN graphs, $k=30$. Run-times $T_1,T_2$ are in seconds. For details, please see Section~\ref{subsec:classification} for details.}  
\begin{tabular}{r|ccc|ccc|}
\multicolumn{1}{c}{} &  \multicolumn{3}{c}{NuMax} &  
\multicolumn{3}{c}{\methodname} \\
\cline{2-7}
&
\multicolumn{1}{c}{$T_{1}$} &
\multicolumn{1}{c}{$T_{2}$} &
\multicolumn{1}{c|}{acc.\%} &
\multicolumn{1}{c}{$T_1$} &
\multicolumn{1}{c}{$T_2$} &
\multicolumn{1}{c|}{acc.\%}  \\ \cline{2-7}
\textsf{Glass}   &   32.04  & 0.1    & 78.2      &  0.004  &   0.006  & 79.9     \\
\textsf{Ion.}   &    24.0 & 0.35     &  33.5    & 0.01    & 0.02    & 64.13      \\
\textsf{Iris}   & 32.7     &  0.029   & 56.4      &   0.004 &   0.006  & 65.6    \\
\textsf{Pima}   & 4\,247.8    &  4.72     & 61.43       & 0.01    & 0.05    &  61.76     \\
\textsf{Veh.}    &   101.9  & 8.80     &  74.12    &   0.007 &   0.84 &  74.22  \\
\textsf{Wine}   &  38.7   & 0.04    &  51.2     &  0.007  & 0.006    & 55.66      \\
\cline{2-7}
\end{tabular}
\end{table}

\section{Conclusion}
\label{sec:concl}
\spara{Summary.}  In this paper we propose \methodname, a novel data-aware near-isometric linear dimensionality method. Our method is scalable, amenable to distributed implementation, and comes with strong theoretical guarantees. It leverages the underlying dataset geometry  while  maintaining well {\em all} pairwise distances between points. Our experimental results show that \methodname scales significantly better than NuMax, a state-of-the-art method \cite{hedge2015}.  Furthermore, it can be used as a preprocessing step in nearest-neighbor related applications, such as classification using $k$-nearest neighbor graphs and approximate nearest neighbor queries, in order to improve their computational efficiency while maintaining the output quality.

\spara{Open Problems.}  An interesting general direction is to design better data-aware near-isometric methods. We conjecture that the following problem is NP-hard. Specifically, given a dataset $T \subset \field{R}^d$, target dimension $r$ and target distortion $\delta$, does there exist a matrix $\Pi\in\field{R}^{r\times d}$ with distortion at most $\delta$ with respect to $T$?

\section*{Acknowledgment}
 
We would like to thank Chinmay Hedge for sharing the original NuMax code.

\bibliographystyle{abbrv}
\bibliography{ref}

\end{document}